\documentclass[final,onefignum,onetabnum]{siamonline190516}

\usepackage{graphicx,subcaption}
\usepackage{array}
\usepackage{color,soul}
\newcolumntype{P}[1]{>{\centering\arraybackslash}p{#1}}

\usepackage{lipsum}
\usepackage{amsfonts}
\usepackage{graphicx}
\usepackage{epstopdf}
\usepackage{ulem}
\usepackage{algorithmic}
\usepackage{mathtools}
\usepackage{slashbox}

\ifpdf
  \DeclareGraphicsExtensions{.eps,.pdf,.png,.jpg}
\else
  \DeclareGraphicsExtensions{.eps}
\fi

\usepackage{enumitem}
\setlist[enumerate]{leftmargin=.5in}
\setlist[itemize]{leftmargin=.5in}

\newsiamremark{remark}{Remark}
\newsiamremark{hypothesis}{Hypothesis}
\crefname{hypothesis}{Hypothesis}{Hypotheses}
\newsiamthm{claim}{Claim}

\headers{SPSD Riemannian Geometry with Application to DA}{O. Yair, A. Lahav, R. Talmon}

\title{Symmetric Positive Semi-definite Riemannian Geometry with Application to Domain Adaptation\thanks{Submitted to the editors DATE.
\funding{This work was funded by the European Union’s Horizon 2020 research grant agreement 802735.}}}

\author{Or Yair\thanks{Viterbi Faculty of Electrical Engineering, Technion, Israel Institute of Technology (\email{oryair@campus.technion.ac.il}), (\email{ronen@ef.technion.ac.il}.)
O. Yair is supported by the Adams Fellowship Program of the Israel Academy of Sciences and Humanities}
\and Almog Lahav\footnotemark[2]
\and Ronen Talmon\footnotemark[2]}

\usepackage{amsopn}

\ifpdf
\hypersetup{
  pdftitle={Symmetric Positive Semi-definite Riemannian Geometry with Application to Domain Adaptation},
  pdfauthor={TODO}
}
\fi

\usepackage{enumitem} 
\newlist{propenum}{enumerate}{1} 
\setlist[propenum]{label=\bfseries Property \arabic*., 
                   ref=\arabic*, wide}

\begin{document}

\maketitle

\begin{abstract}
    In this paper, we present new results on the Riemannian geometry of symmetric positive semi-definite (SPSD) matrices. 
    First, based on an existing approximation of the geodesic path, we introduce approximations of the logarithmic map, the exponential maps, and Parallel Transport (PT). Second, we derive a canonical representation for a set of SPSD matrices. 
    Based on these results, we propose an algorithm for Domain Adaptation (DA) and demonstrate its performance in two applications: fusion of hyper-spectral images and motion recognition. 
\end{abstract}




\section{Introduction}
    %
Recent technological advances give rise to the collection and storage of massive complex datasets. 
These datasets are often high dimensional and multimodal, calling for the development of informative representations. 
Since such complex data typically do not live in a Euclidean space, standard linear analysis techniques applied directly to the data are inappropriate. 
Consequently, analysis techniques based on Riemannian geometry have attracted significant research attention. 
While the basic mathematical operations used for data analysis, e.g., addition, subtraction, and comparison, are straight-forward in the Euclidean space, they are often non-trivial or intractable in particular Riemannian spaces.

A considerable amount of literature has been published on the Riemannian geometry of Symmetric Positive Definite (SPD) matrices, where it was shown to be useful for various applications, e.g., in computer vision, medical data analysis, and machine learning \cite{pennec2006riemannian,barachant2013classification,jayasumana2015kernel,huang2015log,bacak2016second, rodrigues2017dimensionality,rodrigues2018riemannian,bergmann2018priors}. 
For example, in \cite{pennec2006riemannian}, Pennec et al. introduced the use of the affine-invariant metric, facilitating closed-form expressions of the exponential and logarithmic maps, in medical imaging. In \cite{barachant2013classification}, Barachant et al. proposed an algorithm based on the Riemannian distance and the estimation of the Riemannian mean \cite{moakher2005differential} in the manifold of SPD matrices for Brain Computer Interface (BCI). 
The PT on the SPD manifold, which has a closed-form expression, was used in \cite{yair2019parallel} for DA.
Similar geometric operations in other Riemannian spaces have been developed as well, e.g., on the Grassmann manifold \cite{absil2004riemannian} and on the Stiefel manifold \cite{edelman1998geometry}, and were shown to be beneficial for a wide variety of data analysis tasks, e.g. \cite{shrivastava2014unsupervised}.

In this paper, we consider the Riemannian geometry of symmetric positive semi-definite (SPSD) matrices. 
Formally, let $\mathcal{S}_{d,r}^{+}$ denote the set of $d \times d$ SPSD matrices with a fixed rank $r < d$. 
Based on the eigenvalue decomposition, it can be shown that any SPSD matrix $\boldsymbol{C} \in \mathcal{S}_{d,r}^{+}$ can be represented by
\[
\boldsymbol{C}=\boldsymbol{G}\boldsymbol{P}\boldsymbol{G}^{T}
\]
where $\boldsymbol{G}\in \mathbb{R}^{d\times r}$ has orthonormal columns representing a point on the Grassmann manifold and $\boldsymbol{P}\in \mathbb{R}^{r\times r}$ is an SPD matrix. 
This geometry extends the Riemannian geometry of SPD matrices. 
In addition, it facilitates the analysis of a larger pool of data features.
For example, the SPD geometry supports the analysis of only full-rank covariance matrices.
However, it is well known that in many real-world problems, this is not the case. Often, high-dimensional data such as gene expression data \cite{kapur2016gene} and hyper-spectral imaging data \cite{HyperspectralLowRank, niu2016hyperspectral} have an intrinsic low-rank structure, and therefore, the associated covariance matrices are not full-rank. 
In addition to supporting low-rank covariance matrices, in contrast to the SPD geometry, the SPSD geometry applies to a wide variety of kernels, graph Laplacians, and similarity matrices, which are common data features in contemporary data analysis.

Despite the high relevance to many data analysis techniques, the usage of the Riemannian geometry of SPSD matrices has thus far been limited, since it lacks several pivotal components.
First, there is no available explicit expression for the geodesic path in $\mathcal{S}_{d,r}^{+}$ connecting two SPSD matrices. As a consequence, there is no definitive expression for the Riemannian distance between two SPSD matrices, which is typically defined as the arc length of the geodesic path. In addition, basic operations such as the logarithmic and the exponential maps, which are derived from the geodesic path, are undefined. Second, the representation of $\boldsymbol{C}\in \mathcal{S}^{+}_{d,r}$ by a pair $(\boldsymbol{G},\boldsymbol{P})$ is not unique, posing challenges when jointly analyzing multiple SPSD matrices.
These missing components led to an avenue of research, where full-rank structure is imposed by adding a scalar matrix to each of the given low-rank matrices \cite{wang2012covariance,fang2018new}. Essentially, this approach ``artificially'' transforms the SPSD geometry into the SPD geometry by introducing a component that does not stem from the data.

Instead, here we propose to extend the Riemannian geometry of SPSD matrices head-on.
Our developments largely rely on the work of Bonnabel and Sepulchre \cite{bonnabel2010riemannian}, where an approximation of the geodesic path on the SPSD manifold was presented, giving rise to a meaningful measure of proximity between two SPSD matrices, and on the work of Bonnabel et al. \cite{bonnabel2013rank}, where a rank-preserving mean of a set of fixed-rank SPSD matrices was defined. 
First, based on the approximation of the geodesic path in $\mathcal{S}_{d,r}^{+}$ \cite{bonnabel2010riemannian}, we introduce an approximation of the logarithmic and exponential maps. Second, we present {\color{black}an approximation of PT on $\mathcal{S}_{d,r}^{+}$}. Finally, using the mean of SPSD matrices proposed in \cite{bonnabel2013rank}, we derive a canonical representation of a set of SPSD matrices. 

Based on the developed mathematical infrastructure for the analysis of SPSD matrices with a fixed rank, we address the problem of DA.
%
Often, due to the inherent heterogeneity of many types of datasets, useful representations usually cannot be achieved simply by considering the union of multiple datasets. 
We present an algorithm for DA, which is based on the proposed canonical representation and PT on $\mathcal{S}_{d,r}^{+}$ and facilitates an informative representation of multiple heterogeneous datasets. We showcase the performance of our algorithm in two applications. First, we demonstrate fusion of hyper-spectral images collected by airborne sensors, which allows high-quality categorization of land-covers in one image by training a classifier on another image. Second, we show accurate motion recognition based on recordings of motions, which is actor-independent, i.e., independent of the actor executing these motions.

The remainder of the paper is organized as follows.
In \cref{sec:Prel} we present preliminaries on the Riemannian manifolds which are relevant to our work: the manifold of SPD matrices $\mathcal{P}_d$, the Grassmann manifold $\mathcal{G}_{d,r}$ and the manifold of SPSD matrices with a fixed rank $\mathcal{S}_{d,r}^+$. 
In \cref{sec:PTandGrassTransportation}, we describe a particular transportation map on $\mathcal{P}_d$ and $\mathcal{G}_{d,r}$ that is derived from PT.
\Cref{sec:SpsdTransportation} presents our approximations for the logarithmic and the exponential maps on $\mathcal{S}_{d,r}^+$, the PT-driven transportation map on $\mathcal{S}_{d,r}^+$, and a canonical representation for a set of SPSD matrices. Next, we propose a new DA algorithm in \cref{sec:DA}. \Cref{sec:Experiments} consists of two applications of the proposed DA algorithm to real data. Finally, \cref{sec:conclusions} concludes the paper.

\section{Preliminaries}
\label{sec:Prel}
In this section, we briefly describe several known properties of the manifold of SPD matrices $\mathcal{P}_d$, the Grassmann manifold $\mathcal{G}_{d,r}$, and the manifold of SPSD matrices $\mathcal{S}_{d,r}^+$, which will be extensively used throughout the paper.
First, we formally denote the following sets:
\begin{itemize}
\item $\mathcal{P}_{d}$ -- The set of $d\times d$ SPD matrices.
\item $\mathcal{S}_{d,r}^{+}$ -- The set of $d\times d$ SPSD matrices
with rank $r< d$.
\item $\mathcal{G}_{d,r}$ -- The set of $r$-dimensional subspaces of
$\mathbb{R}^{d}$.
\item $\mathcal{V}_{d,r}$ -- The set of $d\times r$ matrices with orthonormal
columns: $\boldsymbol{U}^{T}\boldsymbol{U}=\boldsymbol{I}_{r}$ for $\boldsymbol{U} \in \mathcal{V}_{d,r}$.
\item $\mathcal{O}_{d}$ -- The set of $d\times d$ orthogonal matrices
$\mathcal{O}_{d}\cong\mathcal{V}_{d,d}$.
\end{itemize}
In addition, given a manifold $\mathcal{M}$ with its Riemannian geodesic distance $d_R$, we denote the Fr\'echet (Karcher) mean $\overline{x}$ of the set $\left\{ x_{i}\in\mathcal{M}\right\} _{i}$ by
\[
\overline{x}=M\left(\left\{ x_{i}\right\} \right)\coloneqq\arg\min_{x\in\mathcal{M}}\sum_{i}d_{R}^{2}\left(x,x_{i}\right)
\]

%

\subsection{The manifold of SPD matrices $\mathcal{P}_d$}
The matrix $\boldsymbol{P} \in \mathbb{R}^{d\times d}$ is an SPD matrix if it is symmetric and all of its eigenvalues are strictly positive.
Denote the set of all $d \times d$ SPD matrices by
\[
\mathcal{P}_{d}=\left\{ \boldsymbol{P}\in\mathbb{R}^{d\times d}:\boldsymbol{P}=\boldsymbol{P}^{T},\boldsymbol{P}\succ0\right\} 
\]
The set $\mathcal{P}_d$ can be embedded in a $\frac{1}{2}d\left(d+1\right)$ dimensional space, that is
\[
\dim\left(\mathcal{P}_{d}\right)=\frac{1}{2}d\left(d+1\right)
\]
The tangent space $\mathcal{T}_{\boldsymbol{P}}\mathcal{P}_d$ at any point $\boldsymbol{P}\in\mathcal{P}_d$ is the set of all symmetric matrices
\[
\mathcal{T}_{\boldsymbol{P}}\mathcal{P}_d=\left\{ \boldsymbol{S}\in\mathbb{R}^{d\times d}:\boldsymbol{S}=\boldsymbol{S}^{T}\right\} 
\]
The affine invariant metric (inner product) in the tangent space $\mathcal{T}_{\boldsymbol{P}}\mathcal{P}_{d}$ is given by
\begin{equation}\label{eq:spd_inner}
\left\langle \boldsymbol{S}_{1},\boldsymbol{S}_{2}\right\rangle _{\boldsymbol{P}}=\left\langle \boldsymbol{P}^{-\frac{1}{2}}\boldsymbol{S}_{1}\boldsymbol{P}^{-\frac{1}{2}},\boldsymbol{P}^{-\frac{1}{2}}\boldsymbol{S}_{2}\boldsymbol{P}^{-\frac{1}{2}}\right\rangle 
\end{equation}
for any $\boldsymbol{S}_{1},\boldsymbol{S}_{2}\in\mathcal{T}_{\boldsymbol{P}}\mathcal{P}_{d}$, where $\left\langle \cdot,\cdot\right\rangle $ is the standard Euclidean inner product given by $\left\langle \boldsymbol{A},\boldsymbol{B}\right\rangle =\text{Tr}\left\{ \boldsymbol{A}^{T}\boldsymbol{B}\right\} $.

The set $\mathcal{P}_d$ equipped with the affine invariant metric \cref{eq:spd_inner} gives rise to a Riemannian manifold.
Below, we outline the main properties of this manifold. For more details, we refer the readers to \cite{bhatia2009positive}. 
\begin{itemize}
\item The geodesic path from $\boldsymbol{P}_{1}\in\mathcal{P}_{d}$ to
$\boldsymbol{P}_{2}\in\mathcal{P}_{d}$ can be parametrized by
\begin{equation}
    \label{eq:SpdGeodesic}    
\gamma_{\boldsymbol{P}_{1}\rightarrow\boldsymbol{P}_{2}}^{\mathcal{P}}\left(t\right)=\boldsymbol{P}_{1}^{\frac{1}{2}}\left(\boldsymbol{P}_{1}^{-\frac{1}{2}}\boldsymbol{P}_{2}\boldsymbol{P}_{1}^{-\frac{1}{2}}\right)^{t}\boldsymbol{P}_{1}^{\frac{1}{2}},\qquad t\in\left[0,1\right]
\end{equation}

\item The arc length of the geodesic path defines an affine invariant distance and is explicitly given by
\[
d_{\mathcal{P}}^{2}\left(\boldsymbol{P}_{1},\boldsymbol{P}_{2}\right)=\left\Vert \log\left(\boldsymbol{P}_{1}^{-\frac{1}{2}}\boldsymbol{P}_{2}\boldsymbol{P}_{1}^{-\frac{1}{2}}\right)\right\Vert _{F}^{2}=\sum_{i=1}^{d}\log^{2}\left(\lambda_{i}\left(\boldsymbol{P}_{1}^{-1}\boldsymbol{P}_{2}\right)\right)
\]
where $\lambda_{i}\left(\boldsymbol{A}\right)$ is the $i$th eigenvalue
of the matrix $\boldsymbol{A}$, and $\left\Vert \cdot\right\Vert _{F}$ is the Frobenius norm.

\item The exponential map from the point $\boldsymbol{P}\in\mathcal{P}_{d}$ at the direction $\boldsymbol{S}\in\mathcal{T}_{\boldsymbol{P}}\mathcal{P}_{d}$ is given by
\begin{equation}
    \label{eq:SpdExp}
\mathcal{P}_{d} \ni \text{Exp}_{\boldsymbol{P}}\left(\boldsymbol{S}\right)=\boldsymbol{P}^{\frac{1}{2}}\exp\left(\boldsymbol{P}^{-\frac{1}{2}}\boldsymbol{S}\boldsymbol{P}^{-\frac{1}{2}}\right)\boldsymbol{P}^{\frac{1}{2}}
\end{equation}
\item The logarithmic map, which is the inverse of the exponential map, is given by
\begin{equation}
    \label{eq:SpdLog}
\mathcal{T}_{\boldsymbol{P}}\mathcal{P}_{d} \ni \text{Log}_{\boldsymbol{P}}\left(\boldsymbol{P}_{0}\right)=\boldsymbol{P}^{\frac{1}{2}}\log\left(\boldsymbol{P}^{-\frac{1}{2}}\boldsymbol{P}_{0}\boldsymbol{P}^{-\frac{1}{2}}\right)\boldsymbol{P}^{\frac{1}{2}}
\end{equation}
for any $\boldsymbol{P},\boldsymbol{P}_0 \in \mathcal{P}_d$.
\item The PT $\Gamma_{\boldsymbol{P}_{1}\to\boldsymbol{P}_{2}}:\mathcal{T}_{\boldsymbol{P}_{1}}\mathcal{P}_{d}\to\mathcal{T}_{\boldsymbol{P}_{2}}\mathcal{P}_{d}$
of the tangent vector $\boldsymbol{S}\in\mathcal{T}_{\boldsymbol{P}_{1}}\mathcal{P}_{d}$
to $\mathcal{T}_{\boldsymbol{P}_{2}}\mathcal{P}_{d}$, is given by
\[
\Gamma_{\boldsymbol{P}_{1}\rightarrow\boldsymbol{P}_{2}}\left(\boldsymbol{S}\right)=\boldsymbol{E}\boldsymbol{S}\boldsymbol{E}^{T},\qquad\boldsymbol{E}=\left(\boldsymbol{P}_{2}\boldsymbol{P}_{1}^{-1}\right)^{\frac{1}{2}}
\]
\item Given a set of SPD matrices $\left\{ \boldsymbol{P}_{i}\in\mathcal{P}_{d}\right\} _{i}$, a useful Euclidean vector approximation in the tangent space $\mathcal{T}_{\overline{\boldsymbol{P}}}\mathcal{P}$, where $\overline{\boldsymbol{P}}=M\left(\left\{ \boldsymbol{P}_{i}\right\} \right)$, is given by
\[
d_{\mathcal{P}}\left(\boldsymbol{P}_{i},\boldsymbol{P}_{j}\right)\underset{\geq}{\approx}\left\Vert \widehat{\boldsymbol{S}}_{i}-\widehat{\boldsymbol{S}}_{j}\right\Vert _{F}
\]
where $\widehat{\boldsymbol{S}}_i=\overline{\boldsymbol{P}}^{-\frac{1}{2}}\text{Log}_{\overline{\boldsymbol{P}}}\left(\boldsymbol{P}_i\right)\overline{\boldsymbol{P}}^{-\frac{1}{2}}=\log\left(\overline{\boldsymbol{P}}^{-\frac{1}{2}}\boldsymbol{P}_i\overline{\boldsymbol{P}}^{-\frac{1}{2}}\right)$
\end{itemize}

Given a set of SPD matrices $\left\{ \boldsymbol{P}_{i}\in\mathcal{P}_{d}\right\} _{i}$, \cref{alg:SPD_mean} can be used to obtain the Riemannian SPD mean $\overline{\boldsymbol{P}}=M\left(\left\{ \boldsymbol{P}_{i}\right\} \right)$.

\begin{algorithm}
\caption{SPD Mean}
    \label{alg:SPD_mean}

\textbf{\uline{Input}}\textbf{:} A set of SPD matrices  $\left\{ \boldsymbol{P}_{i}\in\mathcal{P}_{d}\right\} _{i=1}^{N}$

\textbf{\uline{Output}}\textbf{:} The Riemannian mean $\overline{\boldsymbol{P}}=M\left(\left\{ \boldsymbol{P}_{i}\right\} \right)$
\begin{enumerate}
\item \textbf{set} $\overline{\boldsymbol{P}}\leftarrow\frac{1}{N}\sum_{i=1}^{N}\boldsymbol{P}_{i}$
\item \textbf{do}
\begin{enumerate}
\item $\overline{\boldsymbol{S}}\leftarrow\frac{1}{N}\sum_{i=1}^{N}\text{Log}_{\overline{\boldsymbol{P}}}\left(\boldsymbol{P}_{i}\right)$
\hfill $\rhd$ using \cref{eq:SpdLog}
\item $\overline{\boldsymbol{P}}\leftarrow\text{Exp}_{\overline{\boldsymbol{P}}}\left(\overline{\boldsymbol{S}}\right)$
\hfill $\rhd$ using \cref{eq:SpdExp}
\end{enumerate}
\textbf{while $\left\Vert \overline{\boldsymbol{S}}\right\Vert _{F}>\epsilon$}
\end{enumerate}
\end{algorithm}

\subsection{The Grassmann manifold $\mathcal{G}_{d,r}$}
Let
\begin{equation}\label{eq:g_fiber}
\left[\boldsymbol{Q}\right]\coloneqq\left\{ \boldsymbol{Q}\left[\begin{matrix}\boldsymbol{Q}_{r} & \boldsymbol{0}\\
\boldsymbol{0} & \boldsymbol{Q}_{d-r}
\end{matrix}\right]\,\bigg|\,\boldsymbol{Q}\in\mathcal{O}_{d},\boldsymbol{Q}_{r}\in\mathcal{O}_{r},\boldsymbol{Q}_{d-r}\in\mathcal{O}_{d-r}\right\} 
\end{equation}
be the equivalence class of all orthogonal matrices such that their $r$ 
leftmost columns span the same subspace. If $\boldsymbol{Q}_{1},\boldsymbol{Q}_{2}\in\left[\boldsymbol{Q}\right]$,
that is, the $r$ leftmost columns have the same span, we denote the equivalence relation by
\begin{equation}
    \label{eq:EquivalentClass}
\boldsymbol{Q}_{1},\boldsymbol{Q}_{2}\in\left[\boldsymbol{Q}\right] \iff \boldsymbol{Q}_{1}\sim\boldsymbol{Q}_{2}.
\end{equation}

For convenience, when considering only
the $r$ leftmost columns of $\boldsymbol{Q}\in\mathcal{O}_{d}$ we sometimes use $\boldsymbol{G}\in\mathcal{V}_{d,r}$ (and similarly $\left[\boldsymbol{G}\right]$)
instead of $\boldsymbol{Q}$ (and $\left[\boldsymbol{Q}\right]$), and we will state the dimensions
explicitly when necessary. This ``thin representation'', using $\boldsymbol{G}$ instead of $\boldsymbol{Q}$, gives rise to economic implementations of most of the operations detailed below. 

Let $\mathcal{G}_{d,r}=\left\{ \left[\boldsymbol{Q}\right]\right\}$ be the set of all $r$-dimensional subspaces of
$\mathbb{R}^{d}$, where $\left[\boldsymbol{Q}\right]$ represents any unique $r$-dimensional span as in \cref{eq:g_fiber}.
It can also be viewed as the quotient space
$$\mathcal{G}_{d,r}=\mathcal{O}_{d}\big/\left(\mathcal{O}_{r}\times\mathcal{O}_{d-r}\right)$$

Following \cite{edelman1998geometry}, for computational purposes, we usually consider a single matrix, either $\boldsymbol{G}\in\mathcal{V}_{d,r}$ or $\boldsymbol{Q}\in\mathcal{O}_{d}$, to represent the entire equivalence class $\left[\boldsymbol{Q}\right]$.
Throughout the paper, when considering multiple points (subspaces) on the Grassmann manifold, we assume that the principal angels between those subspaces are strictly smaller than $\frac{\pi}{2}$.

The set $\mathcal{G}_{d,r}$ can be embedded in a $r\left(d-r\right)$ dimensional space, that is
\[
\dim\left(\mathcal{G}_{d,r}\right)=r\left(d-r\right)
\]
%
%
%
The tangent space $\mathcal{T}_{\boldsymbol{Q}}\mathcal{G}_{d,r}$ at $[\boldsymbol{Q}] \in \mathcal{G}_{d,r}$, represented by the orthogonal matrix $\boldsymbol{Q} \in \mathcal{O}_d$, is given by
\[
\mathcal{T}_{\boldsymbol{Q}}\mathcal{G}_{d,r}=\left\{ \boldsymbol{\Delta}\in\mathbb{R}^{d\times d}\bigg|\boldsymbol{\Delta}=\boldsymbol{Q}\boldsymbol{B}^{\text{skew}}\right\} 
\]
where $\boldsymbol{B}^{\text{skew}}=\left[\begin{matrix}\boldsymbol{0} & -\boldsymbol{B}^{T}\\
\boldsymbol{B} & \boldsymbol{0}
\end{matrix}\right]$ for any $\boldsymbol{B}\in\mathbb{R}^{(d-r)\times r}$.
For simplicity, the tangent space $\mathcal{T}_{\boldsymbol{Q}}\mathcal{G}_{d,r}$ can be equivalently written as
\[
\mathcal{T}_{\boldsymbol{G}}\mathcal{G}_{d,r}=\left\{ \boldsymbol{G}_{\perp}\boldsymbol{B}\in\mathbb{R}^{d\times r}\bigg|\boldsymbol{B}\in\mathbb{R}^{\left(d-r\right)\times r}\right\} 
\]
where $\boldsymbol{Q}=\left[\begin{matrix}\boldsymbol{G} & \boldsymbol{G}_{\perp}\end{matrix}\right]$, $\boldsymbol{G}\in\mathcal{V}_{d,r}$, and $\boldsymbol{G}_{\perp}\in\mathcal{V}_{d,d-r}$ is the orthogonal complement of $\boldsymbol{G}$.
%
The inner product in $\mathcal{T}_{\boldsymbol{Q}}\mathcal{G}_{d,r}$ is given by
\begin{equation}\label{eq:grass_inner}
\left\langle \boldsymbol{\Delta}_{1},\boldsymbol{\Delta}_{2}\right\rangle _{\boldsymbol{Q}}=\frac{1}{2}\left\langle \boldsymbol{\Delta}_{1},\boldsymbol{\Delta}_{2}\right\rangle =\left\langle \boldsymbol{B}_{1},\boldsymbol{B}_{2}\right\rangle 
\end{equation}
where $\boldsymbol{\Delta}_{i}=\boldsymbol{Q}\boldsymbol{B}_{i}^{\text{skew}}\in\mathcal{T}_{\boldsymbol{Q}}\mathcal{G}_{d,r}$.

The set $\mathcal{G}_{d,r}$ and the inner product \cref{eq:grass_inner} form the Grassmann manifold. 
Below, we outline its main properties.
For more details, we refer the readers to \cite{edelman1998geometry}.



\begin{itemize}

\item The exponential map from the point $\boldsymbol{Q}\in\mathcal{O}_{d}$, which represents the point $\left[\boldsymbol{Q}\right]\in\mathcal{G}_{d,r}$, at the direction $\boldsymbol{\Delta}=\boldsymbol{Q}\boldsymbol{B}^{\text{skew}}\in\mathcal{T}_{\boldsymbol{Q}}\mathcal{G}_{d,r}$ is given by
\begin{equation}\label{eq:GrassExp}
\text{Exp}_{\boldsymbol{Q}}\left(\boldsymbol{\Delta}\right)=\boldsymbol{Q}\exp\left(\boldsymbol{B}^{\text{skew}}\right)
\end{equation}
where $\left[\text{Exp}_{\boldsymbol{Q}}\left(\boldsymbol{\Delta}\right)\right]\in\mathcal{G}_{d,r}$.
For small values of $t$, the curve $\left[\text{Exp}_{\boldsymbol{Q}}\left(t\boldsymbol{\Delta}\right)\right]$
is a geodesic.
Similarly, the exponential map from the point $\boldsymbol{G}\in\mathcal{V}_{d,r}$, which represents the point $\left[\boldsymbol{G}\right]\in\mathcal{G}_{d,r}$, at the direction $\boldsymbol{G}_{\perp}\boldsymbol{B}\in\mathcal{T}_{\boldsymbol{G}}\mathcal{G}_{d,r}$ is given by
\begin{equation}
    \label{eq:GrassExpG}
    \text{Exp}_{\boldsymbol{G}}\left(\boldsymbol{G}_{\perp}\boldsymbol{B}\right)=\left(\boldsymbol{G}\boldsymbol{V}\cos\left(\boldsymbol{\Sigma}\right)+\boldsymbol{U}\sin\left(\boldsymbol{\Sigma}\right)\right)\boldsymbol{V}^{T}
\end{equation}
where $\boldsymbol{G}_{\perp}\boldsymbol{B}=\boldsymbol{U}\boldsymbol{\Sigma}\boldsymbol{V}^T$ is a compact SVD.

\item Given two points $\boldsymbol{G},\boldsymbol{G}_{0}\in\mathcal{V}_{d,r}$, representing the two points $\left[\boldsymbol{G}\right],\left[\boldsymbol{G}_{0}\right]\in\mathcal{G}_{d,r}$, the logarithmic map, which is the inverse of the exponential map, is given by
\begin{equation}
    \label{eq:GrassLogG}
\mathcal{T}_{\boldsymbol{G}}\mathcal{G}_{d,r}\ni\text{Log}_{\boldsymbol{G}}\left(\boldsymbol{G}_{0}\right)=\boldsymbol{U}\arctan\left(\boldsymbol{\Sigma}\right)\boldsymbol{V}^{T}
\end{equation}
where
\[
\left(\boldsymbol{I}-\boldsymbol{G}\boldsymbol{G}^{T}\right)\boldsymbol{G}_{0}\left(\boldsymbol{G}^{T}\boldsymbol{G}_{0}\right)^{-1}=\boldsymbol{U}\boldsymbol{\Sigma}\boldsymbol{V}^{T}
\]
is a compact SVD decomposition.
Let $\boldsymbol{Q}=\left[\begin{matrix}\boldsymbol{G} & \boldsymbol{G}_{\perp}\end{matrix}\right]$ and $\boldsymbol{Q}_{0}=\left[\boldsymbol{G}_{0},\boldsymbol{G}_{0,\perp}\right]$.
The tangent vector in \cref{eq:GrassLogG} can be recast as
\begin{equation}
    \label{eq:GrassLogQ}
\mathcal{T}_{\boldsymbol{Q}}\mathcal{G}_{d,r}\ni\text{Log}_{\boldsymbol{Q}}\left(\boldsymbol{Q}_{0}\right)=\boldsymbol{\Delta}=\boldsymbol{Q}\boldsymbol{B}^{\text{skew}}_0
\end{equation}
where $\boldsymbol{B}_{0}^{\text{skew}}=\left[\begin{matrix}\boldsymbol{0} & -\boldsymbol{B}_{0}^{T}\\
\boldsymbol{B}_{0} & \boldsymbol{0}
\end{matrix}\right]$ and $\text{Log}_{\boldsymbol{G}}\left(\boldsymbol{G}_{0}\right)=\boldsymbol{G}_{\perp}\boldsymbol{B}_{0}$.

\item Given $\boldsymbol{G}_{1}, \boldsymbol{G}_{2}\in\mathcal{V}_{d,r}$, the geodesic between the two points can be computed by
\begin{equation}
    \label{eq:GrassGeodesic}
    \gamma_{\boldsymbol{G}_{1}\to\boldsymbol{G}_{2}}^{\mathcal{G}}\left(t\right)=\text{Exp}_{\boldsymbol{G}_{1}}\left(t\text{Log}_{\boldsymbol{G}_{1}}\left(\boldsymbol{G}_{2}\right)\right),\qquad t\in\left[0,1\right]
\end{equation}
Note that in general $\gamma_{\boldsymbol{G}_{1}\to\boldsymbol{G}_{2}}^{\mathcal{G}}\left(1\right)\sim\boldsymbol{G}_{2}$ but not necessarily $\gamma_{\boldsymbol{G}_{1}\to\boldsymbol{G}_{2}}^{\mathcal{G}}\left(1\right)=\boldsymbol{G}_{2}$.
We note that the expression in \cref{eq:GrassGeodesic} is well defined if all the principal angels between the two subspaces $\left[\boldsymbol{G}_{1}\right]$ and $\left[\boldsymbol{G}_{2}\right]$ are strictly smaller than $\frac{\pi}{2}$.

\item The arc length of the geodesic path between the points $\left[\boldsymbol{G}_{1}\right]\in\mathcal{G}_{d,r}$ and $\left[\boldsymbol{G}_{2}\right]\in\mathcal{G}_{d,r}$ is given by
\[
d_{\mathcal{G}}\left(\boldsymbol{G}_{1},\boldsymbol{G}_{2}\right)=\left\Vert \boldsymbol{\Theta}\right\Vert _{F}
\]
where $\boldsymbol{G}_{1}^{T}\boldsymbol{G}_{2}=\boldsymbol{O}_{1}\left(\cos\boldsymbol{\Theta}\right)\boldsymbol{O}_{2}^{T}$ is an SVD decomposition,
$\boldsymbol{O}_{1},\boldsymbol{O}_2\in\mathcal{O}_{r}$, $\boldsymbol{\Theta}=\text{diag}\left(\left[\theta_{1},\theta_{2},\dots,\theta_{r}\right]\right)$,
and $\{\theta_{i}\}$ are known as the principal angles between the two subspaces
$\left[\boldsymbol{G}_{1}\right]$ and $\left[\boldsymbol{G}_{2}\right]$.


\item The PT of the tangent vector $\boldsymbol{\Delta}=\boldsymbol{Q}\boldsymbol{B}^{\text{skew}}\in\mathcal{T}_{\boldsymbol{Q}}\mathcal{G}_{d,r}$ along the geodesic $\text{Exp}_{\boldsymbol{Q}}\left(t\widetilde{\boldsymbol{\Delta}}\right)$
where $\widetilde{\boldsymbol{\Delta}}=\boldsymbol{Q}\widetilde{\boldsymbol{B}}^{\text{skew}}\in\mathcal{T}_{\boldsymbol{Q}}\mathcal{G}_{d,r}$ is given by 
\begin{align}
\begin{split}
    \label{eq:GrassPT}
\Gamma_{\boldsymbol{Q}\rightarrow\text{Exp}_{\boldsymbol{Q}}\left(t\boldsymbol{\Delta}\right)}\left(\boldsymbol{\Delta}\right) & =\text{Exp}_{\boldsymbol{Q}}\left(t\widetilde{\boldsymbol{\Delta}}\right)\boldsymbol{Q}^{T}\boldsymbol{\Delta}\\
 & =\boldsymbol{Q}\exp\left(t\widetilde{\boldsymbol{B}}^{\text{skew}}\right)\boldsymbol{B}^{\text{skew}}\in\mathcal{T}_{\text{Exp}_{\boldsymbol{Q}}\left(t\boldsymbol{\Delta}\right)}\mathcal{G}_{d,r}
\end{split}
\end{align}

Specifically, if $\widetilde{\boldsymbol{Q}}=\text{Exp}_{\boldsymbol{Q}}\left(\widetilde{\boldsymbol{\Delta}}\right)=\boldsymbol{Q}\exp\left(\widetilde{\boldsymbol{B}}^{\text{skew}}\right)$ we have:
\[
\Gamma_{\boldsymbol{Q}\rightarrow\widetilde{\boldsymbol{Q}}}\left(\boldsymbol{\Delta}\right)=\widetilde{\boldsymbol{Q}}\boldsymbol{B}^{\text{skew}}
\]

\end{itemize}

Given a set of matrices $\left\{ \boldsymbol{Q}_{i}\in\mathcal{O}_{d}\right\} _{i}$, where each represents a point $\left[\boldsymbol{Q}_{i}\right]\in\mathcal{G}_{d,r}$, \cref{alg:Grass_mean} can be used to obtain the Riemannian mean on the Grassmann manifold $\overline{\boldsymbol{Q}}=M\left(\left\{ \boldsymbol{Q}_{i}\right\} \right)$.

Let $\boldsymbol{Q}_{1}\in\mathcal{O}_{d}$ and $\boldsymbol{Q}_{2}\in\mathcal{O}_{d}$
represent two points on $\mathcal{G}_{d,r}$, and let $\boldsymbol{G}_{1}\in\mathcal{V}_{d,r}$
and $\boldsymbol{G}_{2}\in\mathcal{V}_{d,r}$ be their $r$ leftmost columns, respectively. When
considering the Stiefel manifold $\mathcal{V}_{d,r}=\mathcal{O}_{d}/\mathcal{O}_{d-r}$, the closest
point $\widetilde{\boldsymbol{Q}}_{2}$ in
$\left[\boldsymbol{Q}_{2}\right]$ to $\boldsymbol{Q}_{1}$ is given
by
\begin{equation}
    \label{eq:GrassProjectionQ}
\widetilde{\boldsymbol{Q}}_{2}=\Pi_{\boldsymbol{Q}_{1}}\left(\boldsymbol{Q}_{2}\right)\coloneqq\text{Exp}_{\boldsymbol{Q}_{1}}\left(\text{Log}_{\boldsymbol{Q}_{1}}\left(\boldsymbol{Q}_{2}\right)\right)=\text{Exp}_{\boldsymbol{Q}_{1}}\left(\boldsymbol{Q}_{1}\boldsymbol{B}_{2}^{\text{skew}}\right)=\boldsymbol{Q}_{1}\exp\left(\boldsymbol{B}_{2}^{\text{skew}}\right)
\end{equation}
where the logarithmic map is computed using \cref{eq:GrassLogG} and \cref{eq:GrassLogQ}, and the exponential map is computed using \cref{eq:GrassExp}.
See \cref{fig:GrassProj} for illustration.
Using the compact representations $\boldsymbol{G}_1$ and $\boldsymbol{G}_2$, \cref{eq:GrassProjectionQ} can be simply recast as
\begin{equation}
    \label{eq:GrassProjectionG}
\widetilde{\boldsymbol{G}}_{2}=\Pi_{\boldsymbol{G}_{1}}\left(\boldsymbol{G}_{2}\right)\coloneqq\text{Exp}_{\boldsymbol{G}_{1}}\left(\text{Log}_{\boldsymbol{G}_{1}}\left(\boldsymbol{G}_{2}\right)\right)=\boldsymbol{G}_{2}\boldsymbol{O}_{2}\boldsymbol{O}_{1}^{T}
\end{equation}
where $\boldsymbol{G}_{1}^{T}\boldsymbol{G}_{2}=\boldsymbol{O}_{1}\boldsymbol{\Sigma}\boldsymbol{O}_{2}^{T}$
is an SVD decomposition. 
This result will be heavily used in remainder of the paper.

\begin{figure}
    \centering
    \includegraphics[width=0.2\columnwidth]{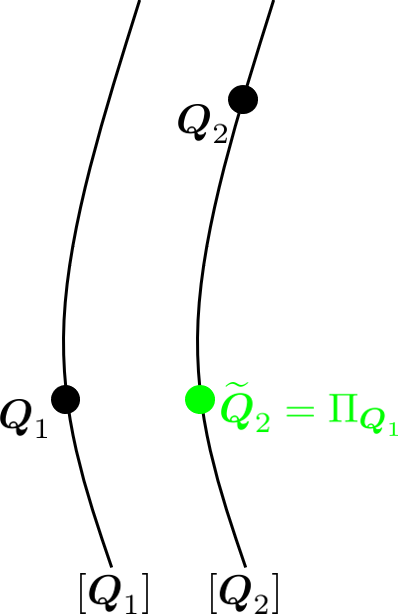}
    \caption{Illustration of the computation of the closest point in $[\boldsymbol{Q}_2]$ to $\boldsymbol{Q}_1$ in \cref{eq:GrassProjectionQ}.}
    \label{fig:GrassProj}
\end{figure}

\begin{algorithm}
\caption{Grassmann Mean}
    \label{alg:Grass_mean}

\textbf{\uline{Input}}\textbf{:} A set of matrices $\left\{ \boldsymbol{G}_{i}\in\mathcal{V}_{d,r}\right\} _{i}$, each represents a point $\left[\boldsymbol{G}_{i}\right]\in\mathcal{G}_{d,r}$

\textbf{\uline{Output}}\textbf{:} The Grassmann mean $\overline{\boldsymbol{G}}=M\left(\left\{ \boldsymbol{G}_{i}\right\} \right)$
\begin{enumerate}
\item \textbf{set} $\overline{\boldsymbol{G}}\leftarrow\boldsymbol{G}_{1}$
\item \textbf{do}
\begin{enumerate}
\item $\overline{\boldsymbol{\Delta}}\leftarrow\frac{1}{N}\sum_{i=1}^{N}\text{Log}_{\overline{\boldsymbol{G}}}\left(\boldsymbol{Q}_{i}\right)$
\hfill $\rhd$ using \cref{eq:GrassLogG}
\item $\overline{\boldsymbol{G}}\leftarrow\text{Exp}_{\overline{\boldsymbol{G}}}\left(\overline{\boldsymbol{\Delta}}\right)$
\hfill $\rhd$ using \cref{eq:GrassExpG}
\end{enumerate}
\textbf{while $\left\Vert \overline{\boldsymbol{\Delta}}\right\Vert _{F}>\epsilon$}
\end{enumerate}
\end{algorithm}

\subsection{The manifold of SPSD matrices $\mathcal{S}_{d,r}^+$}
    \label{sub:SPSD}
The set of all $d\times d$ SPSD matrices with a fixed rank $r<d$ is given by
\[
\mathcal{S}_{d,r}^{+}=\left\{ \boldsymbol{C}\in\mathbb{R}^{d\times d}:\boldsymbol{C}=\boldsymbol{C}^{T},\boldsymbol{C}\succeq0,\text{rank}\left(\boldsymbol{C}\right)=r\right\} 
\]
Since any $\boldsymbol{C}\in\mathcal{S}_{d,r}^{+}$ can be represented by
\[
\boldsymbol{C}=\boldsymbol{G}\boldsymbol{P}\boldsymbol{G}^{T}
\]
where $\boldsymbol{G}\in\mathcal{V}_{d,r}$ and $\boldsymbol{P}\in\mathcal{P}_{r}$, Bonnabel and Sepulchre \cite{bonnabel2010riemannian}
proposed the following structure space representation
\[
\boldsymbol{C}\cong\left(\boldsymbol{G},\boldsymbol{P}\right)
\]
We note that this representation is not unique since
\[
\boldsymbol{C}=\boldsymbol{G}\boldsymbol{P}\boldsymbol{G}^{T}=\left(\boldsymbol{G}\boldsymbol{O}\right)\left(\boldsymbol{O}^{T}\boldsymbol{P}\boldsymbol{O}\right)\left(\boldsymbol{G}\boldsymbol{O}\right)^{T}
\]
and therefore
\[
\boldsymbol{C}\cong\left(\boldsymbol{G}\boldsymbol{O},\boldsymbol{O}^{T}\boldsymbol{P}\boldsymbol{O}\right)
\]
for any $\boldsymbol{O}\in\mathcal{O}_{r}$.
In other words, the set $\mathcal{S}_{d,r}^+$ can be written as the quotient manifold
$$
\mathcal{S}_{d,r}^{+}\cong\left(\mathcal{V}_{d,r}\times\mathcal{P}_{r}\right)\big/\mathcal{O}_{r}
$$
Using the structure space, the set $\mathcal{S}_{d,r}^{+}$ can be embedded in $rd -\frac{1}{2}r\left(r-1\right)$ dimensional space, that is
\begin{align*}
\dim\left(\mathcal{S}^+_{d,r}\right) & =rd -\frac{1}{2}r\left(r-1\right)
\end{align*}
The tangent space $\mathcal{T}_{\left(\boldsymbol{G},\boldsymbol{P}\right)}\mathcal{S}^{+}_{d,r}$ in the structure space is given by
\[
\mathcal{T}_{\left(\boldsymbol{G},\boldsymbol{P}\right)}\mathcal{S}_{d,r}^{+}=\left\{ \left(\boldsymbol{\Delta},\boldsymbol{S}\right):\boldsymbol{\Delta}\in\mathcal{T}_{\boldsymbol{G}}\mathcal{G}_{d,r},\boldsymbol{S}\in\mathcal{T}_{\boldsymbol{P}}\mathcal{P}_{r}\right\} 
\]
The inner product in the tangent space $\mathcal{T}_{\left(\boldsymbol{G},\boldsymbol{P}\right)}\mathcal{S}_{d,r}^{+}$ is given by
\begin{equation}
    \label{eq:SPSD_inner}
\left\langle \left(\boldsymbol{\Delta}_{1},\boldsymbol{S}_{1}\right),\left(\boldsymbol{\Delta}_{2},\boldsymbol{S}_{2}\right)\right\rangle _{\left(\boldsymbol{G},\boldsymbol{P}\right)}=\left\langle \boldsymbol{\Delta}_{1},\boldsymbol{\Delta}_{2}\right\rangle _{\boldsymbol{G}}+k\left\langle \boldsymbol{S}_{1},\boldsymbol{S}_{2}\right\rangle _{\boldsymbol{P}},\qquad k>0
\end{equation}
where $\left(\boldsymbol{\Delta}_{i},\boldsymbol{S}_{i}\right)\in\mathcal{T}_{\left(\boldsymbol{G},\boldsymbol{P}\right)}\mathcal{S}_{d,r}^{+}$.

There is no definitive expression for the geodesic path between two points on the manifold. Bonnabel and Sepulchre \cite{bonnabel2010riemannian} proposed the following approximation in the structure space.
Let $\boldsymbol{C}_{1}\cong\left(\boldsymbol{G}_{1},\boldsymbol{P}_{1}\right)$
and $\boldsymbol{C}_{2}\cong\left(\boldsymbol{G}_{2},\boldsymbol{P}_{2}\right)$ be two points on $\mathcal{S}^{+}_{d,r}$ such that $\boldsymbol{G}_{2}\coloneqq\Pi_{\boldsymbol{G}_{1}}\left(\boldsymbol{G}_{2}\right)$.
Then, the approximate geodesic path between $\boldsymbol{C}_1$ and $\boldsymbol{C}_2$ is given by
\begin{equation}
    \label{eq:SpsdGeodesic}
\widetilde{\gamma}_{\boldsymbol{C}_{1}\rightarrow\boldsymbol{C}_{2}}\left(t\right)=\boldsymbol{G}\left(t\right)\boldsymbol{P}\left(t\right)\boldsymbol{G}^{T}\left(t\right),\qquad t\in\left[0,1\right]
\end{equation}
where $\boldsymbol{G}\left(t\right)=\gamma_{\boldsymbol{G}_{1}\rightarrow\boldsymbol{G}_{2}}^{\mathcal{G}}\left(t\right)$
and $\boldsymbol{P}\left(t\right)=\gamma_{\boldsymbol{P}_{1}\rightarrow\boldsymbol{P}_{2}}^{\mathcal{P}}\left(t\right)$ as in \cref{eq:SpdGeodesic} and \cref{eq:GrassGeodesic}.
In addition, the length of the curve $\widetilde{\gamma}_{\boldsymbol{C}_{1}\rightarrow\boldsymbol{C}_{2}}$
is given by
\begin{equation}
    \label{eq:SpsdMetricApprox}
l^{2}\left(\widetilde{\gamma}_{\boldsymbol{C}_{1}\rightarrow\boldsymbol{C}_{2}}\right)=d_{\mathcal{G}}^{2}\left(\boldsymbol{G}_{1},\boldsymbol{G}_{2}\right)+kd_{\mathcal{P}}^{2}\left(\boldsymbol{P}_{1},\boldsymbol{P}_{2}\right),\qquad k>0
\end{equation}
We note that this is not a distance on $\mathcal{S}_{d,r}^{+}$, since it
does not satisfy the triangle inequality.
For more details on the SPSD manifold, we refer the readers to \cite{bonnabel2010riemannian}.

Given a set of SPSD matrices $\left\{ \boldsymbol{C}_{i}\in\mathcal{S}_{d,r}^{+}\right\} _{i}$, an algorithm to compute a point which admits the desirable property of the geometric mean was proposed in \cite{bonnabel2013rank}. The algorithm is summarized in \cref{alg:SpsdMean}.

We remark that the lack of definitive expression for the geodesic path entails that there is also no definitive expression for the logarithmic and exponential maps.

\begin{algorithm}
\caption{Riemannian Mean of SPSD matrices as proposed in \cite{bonnabel2013rank}}
    \label{alg:SpsdMean}

\textbf{\uline{Input}}\textbf{:} A set of SPSD matrices $\left\{ \boldsymbol{C}_{i}\in\mathcal{S}_{d,r}^{+}\right\} _{i=1}^{N}$

\textbf{\uline{Output}}\textbf{:} The proposed mean $\overline{\boldsymbol{C}}=\overline{\boldsymbol{G}}\overline{\boldsymbol{P}}\overline{\boldsymbol{G}}^{T}$
\begin{enumerate}
\item Obtain the Grassmann mean $\overline{\boldsymbol{G}}$:
\begin{enumerate}
\item Obtain $\boldsymbol{G}_{i}\in\mathcal{V}_{d,r}$, the range of $\boldsymbol{C}_{i}$
(e.g., by using SVD)
\item Compute $\overline{\boldsymbol{G}}\in\mathcal{V}_{d,r}$, the Grassmann mean of
$\left\{ \boldsymbol{G}_{i}\right\} _{i=1}^{N}$
\hfill $\rhd$ using \cref{alg:Grass_mean}
\end{enumerate}
\item Obtain the SPD mean $\overline{\boldsymbol{P}}$:
\begin{enumerate}
\item Compute $\boldsymbol{O}_{i}$ and $\overline{\boldsymbol{O}_{i}}$
using SVD:
\[
\boldsymbol{G}_{i}^{T}\overline{\boldsymbol{G}}=\boldsymbol{O}_{i}\boldsymbol{\Sigma}_{i}\overline{\boldsymbol{O}}_{i}
\]
\item Set $\boldsymbol{P}_{i}=\overline{\boldsymbol{O}}_{i}\boldsymbol{O}_{i}^{T}\boldsymbol{G}_{i}^{T}\boldsymbol{C}_{i}\boldsymbol{G}_{i}\boldsymbol{O}_{i}\overline{\boldsymbol{O}}_{i}^{T}\in\mathcal{P}_{r}$
\item Compute $\overline{\boldsymbol{P}}=M\left(\left\{ \boldsymbol{P}_{i}\right\} \right)$, the SPD mean of $\left\{ \boldsymbol{P}_{i}\right\} _{i=1}^{N}$
\hfill $\rhd$ using \cref{alg:SPD_mean}
\end{enumerate}
\item Set $\overline{\boldsymbol{C}}=\overline{\boldsymbol{G}}\overline{\boldsymbol{P}}\overline{\boldsymbol{G}}^{T}$
\end{enumerate}
\end{algorithm}

\section{Transportation on a Riemannian manifold}
    \label{sec:PTandGrassTransportation}
In this section, we study a transport map of a set of points with respect to two reference points on a Riemannian manifold.
This transportation gives the foundation to the proposed DA, as we will show in the sequel. We begin with a general definition.
\begin{definition}
    \label{def:isometricTransportation}
    Consider a set of points $\mathcal{X}=\left\{ x_{i}\in\mathcal{M}\right\} _{i=1}^{N_{x}}$ on a Riemannian manifold $\mathcal{M}$. Let $\overline{x}=M\left(\mathcal{X}\right)$ be the Riemannian mean of the set, and let $\overline{y}\in\mathcal{M}$ be a target mean.
    
    We call a transport map $\varphi_{\overline{x}\to\overline{y}}:\mathcal{M\to}\mathcal{M}$ an \emph{isometric transport of $\mathcal{X}$ from $\overline{x}$ to $\overline{y}$}, if it satisfies the following two properties.
    \begin{enumerate}
        \item $\varphi_{\overline{x}\to\overline{y}}$ preserves pairwise distances:
        $d\left(x_{i},x_{j}\right)=d\left(\varphi_{\overline{x}\to\overline{y}}\left(x_{i}\right),\varphi_{\overline{x}\to\overline{y}}\left(x_{j}\right)\right)$.
        
        \item The mean of the transported set is $\overline{y}$, that is $M\left(\left\{ \varphi_{\overline{x}\to\overline{y}}\left(\mathcal{X}\right)\right\} \right)=\overline{y}$.
        
    \end{enumerate}
\end{definition}
By the above definition, given such a transport map $\varphi_{\overline{x} \rightarrow \overline{y}}$, any composition of ``rotation'' about $\overline{y}$ (in the Riemannian sense) applied to $\varphi_{\overline{x} \rightarrow \overline{y}}\left(\mathcal{X}\right)$ also satisfies \cref{def:isometricTransportation}.

In order to resolve this degree of freedom, we focus on transport maps defined as follows.
Let $\Gamma_{\overline{x}\to\overline{y}}:\mathcal{T}_{\overline{x}}\mathcal{M}\mathcal{\to}\mathcal{T}_{\overline{y}}\mathcal{M}$ be the PT from $\mathcal{T}_{\overline{x}}\mathcal{M}$ to $\mathcal{T}_{\overline{y}}\mathcal{M}$ on the manifold $\mathcal{M}$.
Based on $\Gamma_{\overline{x}\to\overline{y}}$, we define the transport map $\Gamma_{\overline{x}\to\overline{y}}^{+}:\mathcal{M}\rightarrow\mathcal{M}$ as follows
\[
\widetilde{x}_i=\Gamma^{+}_{\overline{x}\rightarrow\overline{y}}\left(x_i\right)\coloneqq\text{Exp}_{\overline{y}}\left(\Gamma_{\overline{x}\rightarrow\overline{y}}\left(\text{Log}_{\overline{x}}\left(x_{i}\right)\right)\right),
\]
for any $x_i \in \mathcal{X}$.
Namely, the map $\Gamma^{+}_{\overline{x}\rightarrow\overline{y}}$ is a composition of three steps:
\begin{enumerate}
\item Apply the logarithmic map to $x_i$ and obtain the corresponding vector $\xi_{i}\in \mathcal{T}_{\overline{x}}\mathcal{M}$
\[
\xi_{i}=\text{Log}_{\overline{x}}\left(x_{i}\right).
\]
\item Parallel transport $\xi_{i}$ from $\mathcal{T}_{\overline{x}}\mathcal{M}$ to $\mathcal{T}_{\overline{y}}\mathcal{M}$
\[
\widetilde{\xi}_{i}=\Gamma_{\overline{x}\rightarrow\overline{y}}\left(\xi_{i}\right).
\]
\item Apply the exponential map to $\widetilde{\xi}_{i}$ and obtain the point $\widetilde{x}_i \in \mathcal{M}$
\[
\widetilde{x}_{i}=\text{Exp}_{\overline{y}}\left(\widetilde{\xi}_{i}\right).
\]
\end{enumerate}
This transport is derived from PT with one important distinction: while PT maps points from tangent space to tangent space, this transport maps points from the manifold to the manifold.

The extra degree of freedom associated with \cref{def:isometricTransportation} is resolved by considering the map $\Gamma^{+}_{\overline{x}\rightarrow\overline{y}}$, because the PTs $\Gamma_{\overline{x}\rightarrow\overline{y}}$ we consider on the specific manifolds of interest are with respect to the Levi-Civita connection. The Levi-Civita connection is the unique torsion-free metric connection. As a result, such PTs along a curve are torsion-free, i.e., they preserve the inner products on the various tangent spaces, circumventing the ``screw around the curve''.

In the following, we will show that the specifications of $\Gamma_{\overline{x}\to\overline{y}}^{+}$ to the manifolds $\mathcal{P}_d$ and $\mathcal{G}_{d,r}$ satisfy \cref{def:isometricTransportation}. In addition, we will provide compact and closed-form expressions for these transports.

\subsection{$\Gamma^+$ on $\mathcal{P}_d$}
    \label{sub:SpdTransportation}
Let $\mathcal{X}=\left\{\boldsymbol{P}_{i}\in\mathcal{P}_{d}\right\}_{i=1}^{N_{x}}$
be a set of points on $\mathcal{P}_d$ with mean $M\left(\mathcal{X}\right)=\overline{\boldsymbol{P}} \in \mathcal{P}_d$,
and let $\overline{\boldsymbol{R}}\in\mathcal{P}_{d}$ be a target mean.
In \cite{yair2019parallel}, it was shown that $\Gamma_{\overline{\boldsymbol{P}}\to\overline{\boldsymbol{R}}}^{+}:\mathcal{P}_{d}\to\mathcal{P}_{d}$ can be written in a compact
(linear) form
\begin{equation}
    \label{eq:SpdPtPlus}
\Gamma^{+}_{\overline{\boldsymbol{P}}\rightarrow\overline{\boldsymbol{R}}}\left(\boldsymbol{P}_i\right)=\boldsymbol{E}\boldsymbol{P}_i\boldsymbol{E}^{T},
\end{equation}
where
\begin{equation}
    \label{eq:E}
\boldsymbol{E}=\left(\overline{\boldsymbol{R}}\overline{\boldsymbol{P}}^{-1}\right)^{\frac{1}{2}}=\overline{\boldsymbol{P}}^{\frac{1}{2}}\left(\overline{\boldsymbol{P}}^{-\frac{1}{2}}\overline{\boldsymbol{R}}\overline{\boldsymbol{P}}^{-\frac{1}{2}}\right)^{\frac{1}{2}}\overline{\boldsymbol{P}}^{-\frac{1}{2}}.
\end{equation}
Direct computation yields that $\Gamma_{\overline{\boldsymbol{P}}\to\overline{\boldsymbol{R}}}^{+}$ admits the properties of \cref{def:isometricTransportation}.
First, $\Gamma_{\overline{\boldsymbol{P}}\to\overline{\boldsymbol{R}}}^{+}$ is isometric, i.e., it preserves the pairwise distances; for any $\boldsymbol{P}_1,\boldsymbol{P}_2 \in \mathcal{P}_d$ 
\[
d_{\mathcal{P}_{d}}\left(\Gamma^{+}_{\overline{\boldsymbol{P}}\rightarrow\overline{\boldsymbol{R}}}\left(\boldsymbol{P}_{1}\right),\Gamma^{+}_{\overline{\boldsymbol{P}}\rightarrow\overline{\boldsymbol{R}}}\left(\boldsymbol{P}_{2}\right)\right)=d_{\mathcal{P}_{d}}\left(\boldsymbol{E}\boldsymbol{P}_{1}\boldsymbol{E}^{T},\boldsymbol{E}\boldsymbol{P}_{2}\boldsymbol{E}^{T}\right)=d_{\mathcal{P}_{d}}\left(\boldsymbol{P}_{1},\boldsymbol{P}_{2}\right),
\]
exploiting the fact that $d_{\mathcal{P}_{d}}$ is affine invariant.
Second, the mean of the transported set coincides with the target mean, that is
\begin{align}
    \label{eq:SpdMeanTransportation}
\begin{split}
M\left(\Gamma_{\overline{\boldsymbol{P}}\rightarrow\overline{\boldsymbol{R}}}^{+}\left(\mathcal{X}\right)\right) & =M\left(\boldsymbol{E}\mathcal{X}\boldsymbol{E}^{T}\right)\underset{\left(*\right)}{=}\boldsymbol{E}M\left(\mathcal{X}\right)\boldsymbol{E}^{T}\\
 & =\boldsymbol{E}\overline{\boldsymbol{P}}\boldsymbol{E}^{T}\underset{\left(**\right)}{=}\overline{\boldsymbol{P}}^{\frac{1}{2}}\left(\overline{\boldsymbol{P}}^{-\frac{1}{2}}\overline{\boldsymbol{R}}\overline{\boldsymbol{P}}^{-\frac{1}{2}}\right)\overline{\boldsymbol{P}}^{\frac{1}{2}}=\overline{\boldsymbol{R}}
\end{split}
\end{align}
where in $\left(*\right)$ we use the congruence invariance property of geometric
mean (see \cite{bhatia2009positive}), and in $\left(**\right)$ we use \cref{eq:E}.

Suppose that $\overline{\boldsymbol{R}}$ is the Riemannian mean of another set $\mathcal{Y} = \left\{ \boldsymbol{R}_{i}\in\mathcal{P}_{d}\right\} _{i=1}^{N_{x}}$.
In addition to satisfying \cref{def:isometricTransportation}, the transported set $\{\Gamma_{\overline{\boldsymbol{P}}\rightarrow\overline{\boldsymbol{R}}}^{+}\left(\boldsymbol{P}_{i}\right) | \boldsymbol{P}_i \in \mathcal{X}\}_{i=1}^{N_x}$ coincides with $\mathcal{Y}$ under the conditions specified in the following statement.
\begin{proposition}
Let $\left\{ \boldsymbol{P}_{i}\right\} _{i}$ be a set of points on $\mathcal{P}_d$ with the Riemannian mean $\overline{\boldsymbol{P}}$.
Consider the map $t:\mathcal{P}_d \rightarrow \mathcal{P}_d$ defined by
\[
\boldsymbol{R}_{i}=t(\boldsymbol{P}_i)=\boldsymbol{T}\boldsymbol{P}_{i}\boldsymbol{T}^{T}
\]
where $\boldsymbol{T}\in \mathrm{GL}_{d}$.
Let $\overline{\boldsymbol{R}}$ be the Riemannian mean of the resulting set $\left\{ \boldsymbol{R}_{i}\right\}_i$.
The following holds
\[
\Gamma_{\overline{\boldsymbol{P}}\rightarrow\overline{\boldsymbol{R}}}^{+}\left(\boldsymbol{P}_{i}\right)=\boldsymbol{R}_{i},\qquad\forall i
\]
if and only if $\boldsymbol{T}$ is of the form $\boldsymbol{T}=\overline{\boldsymbol{P}}^{\frac{1}{2}}\boldsymbol{B}\overline{\boldsymbol{P}}^{-\frac{1}{2}}$
where either $\boldsymbol{B}\succ0$ or $\boldsymbol{B}\prec0$.
\end{proposition}
See the Supplementary Material (SM) for the proof.

Note that on $\mathcal{P}_d$, one can overload the PT operator $\Gamma_{\overline{\boldsymbol{P}}\to\overline{\boldsymbol{R}}}$ with the manifold transportation $\Gamma_{\overline{\boldsymbol{P}}\to\overline{\boldsymbol{R}}}^{+}$, that is
\begin{equation}
\label{eq:SpdOverload}
\Gamma_{\overline{\boldsymbol{P}}\rightarrow\overline{\boldsymbol{R}}}^{+}\left(\boldsymbol{P}_{i}\right)=\Gamma_{\overline{\boldsymbol{P}}\rightarrow\overline{\boldsymbol{R}}}\left(\boldsymbol{P}_{i}\right)=\boldsymbol{E}\boldsymbol{P}_{i}\boldsymbol{E}^{T}.
\end{equation}

\subsection{$\Gamma^+$ on $\mathcal{G}_{d,r}$}
    \label{sub:grassTransportation}
Let $\mathcal{X}=\left\{ \left[\boldsymbol{Q}\right]_{i}\in\mathcal{G}_{d,r}\right\} _{i=1}^{N_{x}}$ be a set of points on $\mathcal{G}_{d,r}$
with mean $M\left(\mathcal{X}\right)=\left[\overline{\boldsymbol{Q}}\right]$,
and let $\left[\overline{\boldsymbol{V}}\right]\in\mathcal{G}_{d,r}$ be a target mean.
On the Grassmann manifold, we have an equivalent result to \cref{eq:SpdOverload}, giving rise to a closed-form expression of $\Gamma_{\overline{\boldsymbol{Q}}\to\overline{\boldsymbol{V}}}^{+}$.
\begin{proposition}
    \label{prop:GrassOverload}
Let $\overline{\boldsymbol{Q}} \in \left[\overline{\boldsymbol{Q}}\right]$ and $\overline{\boldsymbol{V}}\in\left[\overline{\boldsymbol{V}}\right]$
be two points in $\mathcal{O}_{d}$, such that $\overline{\boldsymbol{V}}=\Pi_{\overline{\boldsymbol{Q}}}\left(\overline{\boldsymbol{V}}\right)$.
Define $\Gamma_{\overline{\boldsymbol{Q}}\to\overline{\boldsymbol{V}}}^{+}:\mathcal{G}_{d,r}\to\mathcal{G}_{d,r}$ by
\begin{equation}
    \label{eq:GrassGammsPlus}
\Gamma_{\overline{\boldsymbol{Q}}\rightarrow\overline{\boldsymbol{V}}}^{+}\left(\boldsymbol{Q}_{i}\right)=\mathrm{Exp}_{\overline{\boldsymbol{V}}}\left(\Gamma_{\overline{\boldsymbol{Q}}\rightarrow\overline{\boldsymbol{V}}}\left(\mathrm{Log}_{\overline{\boldsymbol{Q}}}\left(\boldsymbol{Q}_{i}\right)\right)\right)
\end{equation}
Then
\begin{equation}
\label{eq:GrassPtPlus2}
\Gamma_{\overline{\boldsymbol{Q}}\rightarrow\overline{\boldsymbol{V}}}^{+}\left(\boldsymbol{Q}_{i}\right)\sim\Gamma_{\overline{\boldsymbol{Q}}\rightarrow\overline{\boldsymbol{V}}}\left(\boldsymbol{Q}_{i}\right)=\overline{\boldsymbol{V}}\overline{\boldsymbol{Q}}^{T}\boldsymbol{Q}_{i}
\end{equation}
where $\sim$ is the equivalent class as in \cref{eq:EquivalentClass},
and if $\boldsymbol{Q}_{i}$ is chosen such that $\boldsymbol{Q}_{i}=\Pi_{\overline{\boldsymbol{Q}}}\left(\boldsymbol{Q}_{i}\right)$, then the equivalence becomes equality
\begin{equation}
    \label{eq:GrassPtPlus}
\Gamma_{\overline{\boldsymbol{Q}}\rightarrow\overline{\boldsymbol{V}}}^{+}\left(\boldsymbol{Q}_{i}\right)=\Gamma_{\overline{\boldsymbol{Q}}\rightarrow\overline{\boldsymbol{V}}}\left(\boldsymbol{Q}_{i}\right)=\overline{\boldsymbol{V}}\overline{\boldsymbol{Q}}^{T}\boldsymbol{Q}_{i}
\end{equation}
\end{proposition}
See the SM for the proof.
We remark that in \cref{eq:GrassGammsPlus}, we apply the map $\Gamma_{\overline{\boldsymbol{Q}}\rightarrow\overline{\boldsymbol{V}}}^{+}$ to a matrix $\boldsymbol{Q}_{i} \in \mathcal{O}_d$ in $\left[ \boldsymbol{Q}_i \right]$, rather than to $\left[ \boldsymbol{Q}_i \right]$. Similarly, the range of the map is also written as if it is in $\mathcal{O}_d$ rather than in $\mathcal{G}_{d,r}$. For simplicity, we continue with this slight abuse of notation throughout the paper.

Using \cref{prop:GrassOverload}, we show that $\Gamma_{\overline{\boldsymbol{Q}}\to\overline{\boldsymbol{V}}}^{+}$ satisfies the properties of \cref{def:isometricTransportation}.
First, from \eqref{eq:GrassPtPlus2}, since $\overline{\boldsymbol{Q}},\overline{\boldsymbol{V}}\in\mathcal{O}_{d}$, $\Gamma_{\overline{\boldsymbol{Q}}\to\overline{\boldsymbol{V}}}^{+}$ is a unitary transformation, and therefore it preserves the pairwise distances. That is
\[
d_{\mathcal{G}}\left(\Gamma_{\overline{\boldsymbol{Q}}\rightarrow\overline{\boldsymbol{V}}}^{+}\left(\boldsymbol{Q}_{1}\right),\Gamma_{\overline{\boldsymbol{Q}}\rightarrow\overline{\boldsymbol{V}}}^{+}\left(\boldsymbol{Q}_{2}\right)\right)=d_{\mathcal{G}}\left(\boldsymbol{Q}_{1},\boldsymbol{Q}_{2}\right).
\]
Second, the mean of the transported set coincides with the target mean. Namely
\begin{equation}
    \label{eq:GrassMeanTransportation}
M\left(\Gamma_{\overline{\boldsymbol{Q}}\rightarrow\overline{\boldsymbol{V}}}^{+}\left(\mathcal{X}\right)\right)=M\left(\overline{\boldsymbol{V}}\overline{\boldsymbol{Q}}^{T}\mathcal{X}\right)\underset{\left(*\right)}{=}\overline{\boldsymbol{V}}\overline{\boldsymbol{Q}}^{T}M\left(\mathcal{X}\right)=\overline{\boldsymbol{V}}\overline{\boldsymbol{Q}}^{T}\overline{\boldsymbol{Q}}=\overline{\boldsymbol{V}}
\end{equation}
where $\left(*\right)$ is due to the fact that the mean of a rotated set is the rotated mean (see \cite{bonnabel2013rank}).

In the spirit of the three steps comprising $\Gamma^{+}$, the maps $\Gamma_{\overline{\boldsymbol{P}}\to\overline{\boldsymbol{R}}}^{+}$ on the $\mathcal{P}_d$ and $\Gamma_{\overline{\boldsymbol{Q}}\to\overline{\boldsymbol{V}}}^{+}$ on $\mathcal{G}_{d,r}$ can be recast as three steps defined on the respective manifolds rather than via the tangent planes. On $\mathcal{P}_d$, using \cref{eq:SpdPtPlus} and \cref{eq:E} we can write
\begin{equation}
\Gamma_{\overline{\boldsymbol{P}}\rightarrow\overline{\boldsymbol{R}}}^{+}=f_{\mathcal{P}_d}^{-1}\circ g_{\mathcal{P}_d}\circ f_{\mathcal{P}_d}
\end{equation}
where
\[
f_{\mathcal{P}_d}\left(\boldsymbol{P}_{i}\right)=\overline{\boldsymbol{P}}^{-\frac{1}{2}}\boldsymbol{P}_{i}\overline{\boldsymbol{P}}^{-\frac{1}{2}}
\]
and
\[
g_{\mathcal{P}_d}\left(\boldsymbol{P}_{i}\right)=\left(\overline{\boldsymbol{P}}^{-\frac{1}{2}}\overline{\boldsymbol{R}}\overline{\boldsymbol{P}}^{-\frac{1}{2}}\right)^{\frac{1}{2}}\boldsymbol{P}_{i}\left(\overline{\boldsymbol{P}}^{-\frac{1}{2}}\overline{\boldsymbol{R}}\overline{\boldsymbol{P}}^{-\frac{1}{2}}\right)^{\frac{1}{2}}
\]
Similarly on $\mathcal{G}_{d,r}$, using \cref{eq:GrassPT} we can write
\begin{equation}
    \Gamma_{\overline{\boldsymbol{Q}}\rightarrow\overline{\boldsymbol{V}}}^{+}=f_{\mathcal{G}_{d,r}}^{-1}\circ g_{\mathcal{G}_{d,r}}\circ f_{\mathcal{G}_{d,r}}
\end{equation}
where
\[
f_{\mathcal{G}_{d,r}}\left(\boldsymbol{Q}_{i}\right)=\overline{\boldsymbol{Q}}^{T}\boldsymbol{Q}_{i}
\]
and
\[
g_{\mathcal{G}_{d,r}}\left(\boldsymbol{Q}_{i}\right)=\overline{\boldsymbol{Q}}^{T}\overline{\boldsymbol{V}}\boldsymbol{Q}_{i}
\]
On both manifolds, the map $f$ transports the cloud of points to the identity (so the new mean is $\boldsymbol{I}$), and the map $g$ transports the cloud from the identity to ``$\overline{\boldsymbol{R}}$ over $\overline{\boldsymbol{P}}$'', namely $\overline{\boldsymbol{P}}^{-\frac{1}{2}}\overline{\boldsymbol{R}}\overline{\boldsymbol{P}}^{-\frac{1}{2}}$, on $\mathcal{P}_d$, and to ``$\overline{\boldsymbol{V}}$ over $\overline{\boldsymbol{Q}}$'', namely $\overline{\boldsymbol{Q}}^{-1}\overline{\boldsymbol{V}} = \overline{\boldsymbol{Q}}^{T}\overline{\boldsymbol{V}}$ (since $\overline{\boldsymbol{Q}} \in \mathcal{O}_d$), on $\mathcal{G}_{d,r}$. Finally, $f^{-1}$ maps the cloud to $\overline{\boldsymbol{R}}$ (on $\mathcal{P}_d$) and $\overline{\boldsymbol{V}}$ (on $\mathcal{G}_{d,r}$).

\section{Transportation on $\mathcal{S}_{d,r}^+$}
    \label{sec:SpsdTransportation}
In this section we derive a transportation $\widetilde{\Gamma}^{+}:\mathcal{S}_{d,r}^{+}\rightarrow\mathcal{S}_{d,r}^{+}$ in a similar manner to $\Gamma^{+}$ on $\mathcal{P}_d$ and $\mathcal{G}_{d,r}$, which are presented in \cref{sub:SpdTransportation} and \cref{sub:grassTransportation}, respectively.
Given two points $\overline{\boldsymbol{C}},\overline{\boldsymbol{Y}} \in \mathcal{S}_{d,r}^{+}$, let $\widetilde{\Gamma}^{+}:\mathcal{S}_{d,r}^{+}\rightarrow\mathcal{S}_{d,r}^{+}$ be a composition of three steps: (i) projection to the tangent space $\mathcal{T}_{\overline{\boldsymbol{C}}}\mathcal{S}_{d,r}^{+}$, (ii) transportation between the two tangent spaces $\mathcal{T}_{\overline{\boldsymbol{C}}}\mathcal{S}_{d,r}^{+}\to\mathcal{T}_{\overline{\boldsymbol{Y}}}\mathcal{S}_{d,r}^{+}$, and (iii) projection back from $\mathcal{T}_{\overline{\boldsymbol{Y}}}\mathcal{S}_{d,r}^{+}$ to the manifold $\mathcal{S}_{d,r}^{+}$. 
The implementation used in \cref{sub:SpdTransportation} and \cref{sub:grassTransportation} comprises the logarithmic map, PT and the exponential map. However, the logarithmic and the exponential maps in $\mathcal{S}_{d,r}^{+}$ have no explicit expressions, and there are no existing numerical methods to compute these operators. Therefore, here we propose approximations of these operators, which in turn facilitate the construction of a transportation $\widetilde{\Gamma}^{+}:\mathcal{S}_{d,r}^{+}\rightarrow\mathcal{S}_{d,r}^{+}$, which could be viewed as the counterpart of ${\Gamma}^{+}$ from \cref{sec:PTandGrassTransportation} on $\mathcal{S}_{d,r}^{+}$.
However, in contrast to ${\Gamma}^{+}$, since there is no known method to compute the geodesic distance on $\mathcal{S}_{d,r}^+$, $\widetilde{\Gamma}^{+}$ is not guaranteed to admit the isometry property in \cref{def:isometricTransportation}. Nevertheless, we will show that $\widetilde{\Gamma}^{+}$ is useful for DA, similarly to ${\Gamma}^{+}$ on $\mathcal{P}_d$ \cite{yair2019parallel}.


\subsection{Operations on $\mathcal{S}^+_{d,r}$} 

The maps $\Gamma^+$ on $\mathcal{P}_d$ and $\mathcal{G}_{d,r}$ require the exponential and the logarithmic maps as well as PT, which are derived from the geodesics. 
Here, we present approximations to the logarithmic map{\color{black}, to the exponential map and to PT} on $\mathcal{S}_{d,r}^{+}$ based on the approximation of the geodesic given in \cref{eq:SpsdGeodesic}. The presented approximations make use of the structure space representation introduced in \cref{sub:SPSD}, that is, any SPSD matrix $\boldsymbol{C}_{i}\in\mathcal{S}_{d,r}^{+}$ can be represented as $\boldsymbol{C}_{i}\cong\left(\boldsymbol{G}_{i},\boldsymbol{P}_{i}\right)$ where $\boldsymbol{G}_{i}\in\mathcal{V}_{d,r}$ and $\boldsymbol{P}_{i}\in\mathcal{P}_{r}$.

Formally, given the curve $\widetilde{\gamma}_{\left(\boldsymbol{G}_{1},\boldsymbol{P}_{1}\right)\rightarrow\left(\boldsymbol{G}_{2},\boldsymbol{P}_{2}\right)}\left(t\right)$ between the two points $\left(\boldsymbol{G}_{1},\boldsymbol{P}_{1}\right)$ and $\left(\boldsymbol{G}_{2},\boldsymbol{P}_{2}\right)$ such that $\boldsymbol{G}_{2}=\Pi_{\boldsymbol{G}_{1}}\left(\boldsymbol{G}_{2}\right)$ as in \cref{eq:SpsdGeodesic}, we define an approximate of the logarithmic map $\widetilde{\text{L}}_{\left(\boldsymbol{G}_{1},\boldsymbol{P}_{1}\right)}:\mathcal{S}_{d,r}^{+}\rightarrow\mathcal{T}_{\left(\boldsymbol{G}_{1},\boldsymbol{P}_{1}\right)}\mathcal{S}_{d,r}^{+}$ using the derivative of $\widetilde{\gamma}$ (rather than the geodesic)
\begin{align}
    \label{eq:SpsdLogMap}
\begin{split}
\widetilde{\text{L}}_{\left(\boldsymbol{G}_{1},\boldsymbol{P}_{1}\right)}\left(\boldsymbol{G}_{2},\boldsymbol{P}_{2}\right) & =\dot{\widetilde{\gamma}}_{\left(\boldsymbol{G}_{1},\boldsymbol{P}_{1}\right)\rightarrow\left(\boldsymbol{G}_{2},\boldsymbol{P}_{2}\right)}\left(0\right)\\
 & =\left(\dot{\gamma}_{\boldsymbol{G}_{1}\rightarrow\boldsymbol{G}_{2}}^{\mathcal{G}}\left(0\right),\dot{\gamma}_{\boldsymbol{P}_{1}\rightarrow\boldsymbol{P}_{2}}^{\mathcal{P}}\left(0\right)\right)\\
 & =\left(\text{Log}_{\boldsymbol{G}_{1}}\left(\boldsymbol{G}_{2}\right),\text{Log}_{\boldsymbol{P}_{1}}\left(\boldsymbol{P}_{2}\right)\right)
\end{split}
\end{align}
Accordingly, an approximate of the exponential map $\widetilde{\text{E}}_{\left(\boldsymbol{G}_{1},\boldsymbol{P}_{1}\right)}:\mathcal{T}_{\left(\boldsymbol{G}_{1},\boldsymbol{P}_{1}\right)}\mathcal{S}_{d,r}^{+}\rightarrow\mathcal{S}_{d,r}^{+}$, which is the inverse map of the approximate logarithmic map, is given by
\begin{equation}
\widetilde{\text{E}}_{\left(\boldsymbol{G}_{1},\boldsymbol{P}_{1}\right)}\left(\boldsymbol{\Delta},\boldsymbol{S}\right)=\widetilde{\text{L}}^{-1}\left(\boldsymbol{\Delta},\boldsymbol{S}\right)=\left(\text{Exp}_{\boldsymbol{G}_{1}}\left(\boldsymbol{\Delta}\right),\text{Exp}_{\boldsymbol{P}_{1}}\left(\boldsymbol{S}\right)\right)
\end{equation}
where $\left(\boldsymbol{\Delta},\boldsymbol{S}\right) \in \mathcal{T}_{\left(\boldsymbol{G}_{1},\boldsymbol{P}_{1}\right)}\mathcal{S}_{d,r}^{+}$.

{\color{black}
Lastly, we define the transport $\widetilde{\Gamma}_{\left(\boldsymbol{G}_{1},\boldsymbol{P}_{1}\right)\rightarrow\left(\boldsymbol{G}_{2},\boldsymbol{P}_{2}\right)}:\mathcal{T}_{\left(\boldsymbol{G}_{1},\boldsymbol{P}_{1}\right)}\mathcal{S}_{d,r}^{+}\rightarrow\mathcal{T}_{\left(\boldsymbol{G}_{2},\boldsymbol{P}_{2}\right)}\mathcal{S}_{d,r}^{+}$ by
\begin{equation}
\widetilde{\Gamma}_{\left(\boldsymbol{G}_{1},\boldsymbol{P}_{1}\right)\rightarrow\left(\boldsymbol{G}_{2},\boldsymbol{P}_{2}\right)}\left(\boldsymbol{\Delta},\boldsymbol{S}\right):=\left(\Gamma_{\boldsymbol{G}_{1}\rightarrow\boldsymbol{G}_{2}}\left(\boldsymbol{\Delta}\right),\Gamma_{\boldsymbol{P}_{1}\rightarrow\boldsymbol{P}_{2}}\left(\boldsymbol{S}\right)\right)
\end{equation}
for any $\left(\boldsymbol{\Delta},\boldsymbol{S}\right) \in \mathcal{T}_{\left(\boldsymbol{G}_{1},\boldsymbol{P}_{1}\right)}\mathcal{S}_{d,r}^{+}$.
By definition, the transport $\widetilde{\Gamma}_{\left(\boldsymbol{G}_{1},\boldsymbol{P}_{1}\right)\rightarrow\left(\boldsymbol{G}_{2},\boldsymbol{P}_{2}\right)}$
preserves inner product since
\begin{align*}
 & \left\langle \widetilde{\Gamma}_{\left(\boldsymbol{G}_{1},\boldsymbol{P}_{1}\right)\rightarrow\left(\boldsymbol{G}_{2},\boldsymbol{P}_{2}\right)}\left(\boldsymbol{\Delta}_{1},\boldsymbol{S}_{1}\right),\widetilde{\Gamma}_{\left(\boldsymbol{G}_{1},\boldsymbol{P}_{1}\right)\rightarrow\left(\boldsymbol{G}_{2},\boldsymbol{P}_{2}\right)}\left(\boldsymbol{\Delta}_{2},\boldsymbol{S}_{2}\right)\right\rangle _{\left(\boldsymbol{G}_{2},\boldsymbol{P}_{2}\right)}=\\
 & =\left\langle \left(\Gamma_{\boldsymbol{G}_{1}\rightarrow\boldsymbol{G}_{2}}\left(\boldsymbol{\Delta}_{1}\right),\Gamma_{\boldsymbol{P}_{1}\rightarrow\boldsymbol{P}_{2}}\left(\boldsymbol{S}_{1}\right)\right),\left(\Gamma_{\boldsymbol{G}_{1}\rightarrow\boldsymbol{G}_{2}}\left(\boldsymbol{\Delta}_{2}\right),\Gamma_{\boldsymbol{P}_{1}\rightarrow\boldsymbol{P}_{2}}\left(\boldsymbol{S}_{2}\right)\right)\right\rangle _{\left(\boldsymbol{G}_{2},\boldsymbol{P}_{2}\right)}\\
 & =\left\langle \Gamma_{\boldsymbol{G}_{1}\rightarrow\boldsymbol{G}_{2}}\left(\boldsymbol{\Delta}_{1}\right),\Gamma_{\boldsymbol{G}_{1}\rightarrow\boldsymbol{G}_{2}}\left(\boldsymbol{\Delta}_{2}\right)\right\rangle _{\boldsymbol{G}_{2}}+k\left\langle \Gamma_{\boldsymbol{P}_{1}\rightarrow\boldsymbol{P}_{2}}\left(\boldsymbol{S}_{1}\right),\Gamma_{\boldsymbol{P}_{1}\rightarrow\boldsymbol{P}_{2}}\left(\boldsymbol{S}_{2}\right)\right\rangle _{\boldsymbol{P}_{2}}\\
 & =\left\langle \boldsymbol{\Delta}_{1},\boldsymbol{\Delta}_{2}\right\rangle _{\boldsymbol{G}_{1}}+k\left\langle \boldsymbol{S}_{1},\boldsymbol{S}_{2}\right\rangle _{\boldsymbol{P}_{1}}\\
 & =\left\langle \left(\boldsymbol{\Delta}_{1},\boldsymbol{S}_{1}\right),\left(\boldsymbol{\Delta}_{2},\boldsymbol{S}_{2}\right)\right\rangle _{\left(\boldsymbol{G}_{1},\boldsymbol{P}_{1}\right)}
\end{align*}
In addition, we have that
\begin{align*}
 & \widetilde{\Gamma}_{\left(\boldsymbol{G}_{1},\boldsymbol{P}_{1}\right)\rightarrow\left(\boldsymbol{G}_{2},\boldsymbol{P}_{2}\right)}\left(\dot{\widetilde{\gamma}}_{\left(\boldsymbol{G}_{1},\boldsymbol{P}_{1}\right)\rightarrow\left(\boldsymbol{G}_{2},\boldsymbol{P}_{2}\right)}\left(0\right)\right)=\\
 & =\left(\Gamma_{\boldsymbol{G}_{1}\rightarrow\boldsymbol{G}_{2}}\left(\dot{\gamma}_{\boldsymbol{G}_{1}\rightarrow\boldsymbol{G}_{2}}^{\mathcal{G}}\left(0\right)\right),\Gamma_{\boldsymbol{P}_{1}\rightarrow\boldsymbol{P}_{2}}\left(\dot{\gamma}_{\boldsymbol{G}_{1}\rightarrow\boldsymbol{G}_{2}}^{\mathcal{G}}\left(0\right)\right)\right)\\
 & =\left(\dot{\gamma}_{\boldsymbol{G}_{1}\rightarrow\boldsymbol{G}_{2}}^{\mathcal{G}}\left(1\right),\dot{\gamma}_{\boldsymbol{G}_{1}\rightarrow\boldsymbol{G}_{2}}^{\mathcal{G}}\left(1\right)\right)\\
 & =\dot{\widetilde{\gamma}}_{\left(\boldsymbol{G}_{1},\boldsymbol{P}_{1}\right)\rightarrow\left(\boldsymbol{G}_{2},\boldsymbol{P}_{2}\right)}\left(1\right)
\end{align*}
Thus, $\widetilde{\Gamma}_{\left(\boldsymbol{G}_{1},\boldsymbol{P}_{1}\right)\rightarrow\left(\boldsymbol{G}_{2},\boldsymbol{P}_{2}\right)}$ can be viewed as an approximation of PT.
}

Seemingly, equipped with the three operations defined above, we are ready for the construction of the transportation $\widetilde{\Gamma}^{+}$ on $\mathcal{S}_{d,r}^{+}$. 
However, the expression in \cref{eq:SpsdLogMap} is true only when $\boldsymbol{G}_{2}=\Pi_{\boldsymbol{G}_{1}}\left(\boldsymbol{G}_{2}\right)$.
Therefore, in \cref{sub:Canonical} we verify that this condition is met, and in \cref{sub:gamma_tilde_spsd} we present the construction of $\widetilde{\Gamma}^{+}$.

\subsection{Canonical representation}
    \label{sub:Canonical}
Let $\mathcal{C}=\left\{ \boldsymbol{C}_{i}\in\mathcal{S}_{d,r}^{+}\bigg|\boldsymbol{C}_{i}\cong\left(\boldsymbol{U}_{i},\boldsymbol{T}_{i}\right)\right\} _{i=1}^{N}$ be a set of SPSD matrices with some arbitrary structure space representation $\boldsymbol{C}_{i}\cong\left(\boldsymbol{U}_{i},\boldsymbol{T}_{i}\right)$.
Recall that the structure space representation is not unique, that is $\boldsymbol{C}_{i}\cong\left(\boldsymbol{U}_{i},\boldsymbol{T}_{i}\right)\cong\left(\boldsymbol{U}_{i}\boldsymbol{O}_i^{T},\boldsymbol{O}_i\boldsymbol{T}_{i}\boldsymbol{O}_i^T\right)$ for any $\boldsymbol{O}_i\in\mathcal{O}_{r}$.
Let $\left[\overline{\boldsymbol{G}}\right]=M\left(\left\{ \left[\boldsymbol{U}_{i}\right]\right\} \right)$ be the Grassmann mean of the ranges of the SPSD matrices $\left\{ \boldsymbol{C}_{i}\right\}$.
In this subsection, we reduce the number of degrees of freedom stemming from the possible arbitrary choice of $\left\{ \boldsymbol{O}_i \right\}_{i=1}^N$ by fixing a representative matrix $\overline{\boldsymbol{G}}\in\left[\overline{\boldsymbol{G}}\right]$ (any representative will do)
and deriving a structure space representation $\boldsymbol{C}_{i}\cong\left(\boldsymbol{G}_{i},\boldsymbol{P}_{i}\right)$ such that $\boldsymbol{G}_{i}=\Pi_{\overline{\boldsymbol{G}}}\left(\boldsymbol{G}_{i}\right)$. 
This particular structure space representation allows us to apply \cref{eq:SpsdLogMap} to the set $\mathcal{C}$.

Formally, let $\overline{\boldsymbol{C}}\in \mathcal{S}_{d,r}^{+}$ be the Riemannian mean of $\mathcal{C}$ on $\mathcal{S}_{d,r}^{+}$, and let $\overline{\boldsymbol{G}}\in\left[\overline{\boldsymbol{G}}\right]$ where both means can be obtained by \cref{alg:SpsdMean}.
Set $\overline{\boldsymbol{P}}=\overline{\boldsymbol{G}}^{T}\overline{\boldsymbol{C}}\overline{\boldsymbol{G}}\in \mathcal{P}_{r}$ such that $\overline{\boldsymbol{C}}\cong \left(\boldsymbol{\overline{G}},\overline{\boldsymbol{P}}\right)$,
and consider the following canonical representation
\begin{equation}
    \label{eq:Canonical}
\boldsymbol{C}_{i}\cong \left(\boldsymbol{G}_{i},\boldsymbol{P}_{i}\right)
\end{equation}
where $\boldsymbol{G}_{i}=\Pi_{\overline{\boldsymbol{G}}}\left(\boldsymbol{U}_{i}\right)$, 
and $\boldsymbol{P}_{i}=\boldsymbol{G}_{i}^{T}\boldsymbol{C}_{i}\boldsymbol{G}_{i}$.
This `alignment' procedure is illustrated in \cref{fig:CanonicalRepresentation}.
\begin{figure}
    \centering
    \includegraphics[width=0.4\columnwidth]{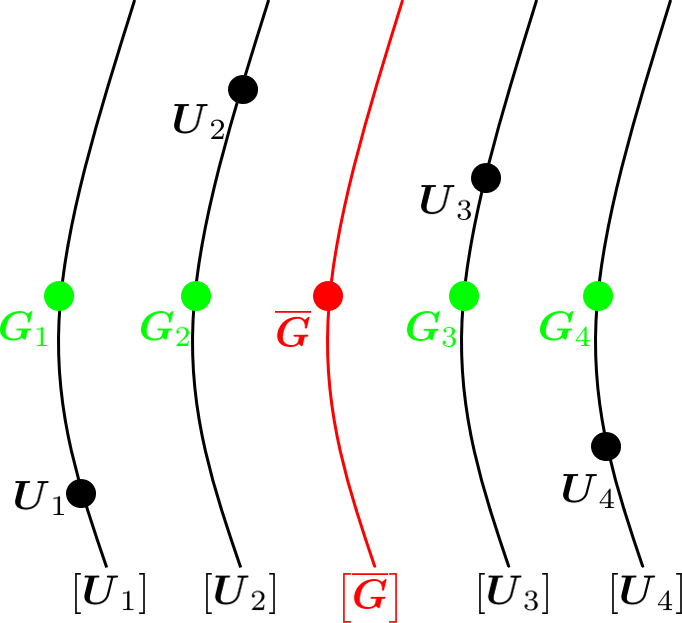}
    \caption{Illustration of the canonical representation.
    $\left[\overline{\boldsymbol{G}}\right]$ is the mean of the set $\left\{ \left[\boldsymbol{U}_{i}\right]\right\} $. We peak a representative $\overline{\boldsymbol{G}}$ (red point) and rotate each $\boldsymbol{U}_i$ to $\boldsymbol{G}_{i}$ (green points) such that $\boldsymbol{G}_{i}=\Pi_{\overline{\boldsymbol{G}}}\left(\boldsymbol{G}_{i}\right)=\Pi_{\overline{\boldsymbol{G}}}\left(\boldsymbol{U}_{i}\right)$.
     }
    \label{fig:CanonicalRepresentation}
\end{figure}

The proposed canonical representation \cref{eq:Canonical} admits the following properties.
First, since $\boldsymbol{G}_{i}\sim\boldsymbol{U}_{i}$ for all $i$, the
mean $\overline{\boldsymbol{G}}$ of $\{\boldsymbol{U}_i\}_i$ coincides with the mean of $\{\boldsymbol{G}_i\}_i$, namely, $\left[\overline{\boldsymbol{G}}\right]=M\left(\left\{ \left[\boldsymbol{U}_{i}\right]\right\} \right)=M\left(\left\{ \left[\boldsymbol{G}_{i}\right]\right\} \right)$.
Second, as shown in \cite{bonnabel2013rank}, it follows from \cref{alg:SpsdMean} that the mean of $\left\{ \boldsymbol{P}_{i}\right\} $ is indeed $\overline{\boldsymbol{P}}=M\left(\left\{ \boldsymbol{P}_{i}\right\} _{i}\right)$.

In addition, after applying the approximate logarithmic map $\widetilde{\mathrm{L}}$ \cref{eq:SpsdLogMap} to the set $\mathcal{C}$, the vectors in the tangent space $\mathcal{T}_{\left(\overline{\boldsymbol{G}},\overline{\boldsymbol{P}}\right)}\mathcal{S}_{d,r}^{+}$ are centered around the origin.
This stems from the above two properties. Specifically, taking the arithmetic mean of vectors obtained by applying $\widetilde{\mathrm{L}}$ to the canonical representations $\{(\boldsymbol{G}_i,\boldsymbol{P}_i)_i\}$ results in
\begin{equation}
\frac{1}{N}
\sum_{i=1}^{N}\widetilde{\text{L}}_{\left(\overline{\boldsymbol{G}},\overline{\boldsymbol{P}}\right)}\left(\boldsymbol{G}_{i},\boldsymbol{P}_{i}\right)=\frac{1}{N}\left(\sum_{i=1}^{N}\log_{\overline{\boldsymbol{G}}}\left(\boldsymbol{G}_{i}\right),\sum_{i=1}^{N}\log_{\overline{\boldsymbol{P}}}\left(\boldsymbol{P}_{i}\right)\right)=\left(\boldsymbol{0},\boldsymbol{0}\right)
\end{equation}
since the arithmetic means of the vectors in $\mathcal{T}_{\overline{\boldsymbol{G}}}\mathcal{G}_{d,r}$ and $\mathcal{T}_{\overline{\boldsymbol{P}}}\mathcal{P}_r$, namely $\sum_{i=1}^{N}\log_{\overline{\boldsymbol{G}}}\left(\boldsymbol{G}_{i}\right)$ and $\sum_{i=1}^{N}\log_{\overline{\boldsymbol{P}}}\left(\boldsymbol{P}_{i}\right)$, lie at the origin, because $\overline{\boldsymbol{G}} = M\left( \{ [\boldsymbol{G_i} ] \} \right)$ and $\overline{\boldsymbol{P}} = M\left( \{ \boldsymbol{P_i} \} \right)$, respectively.
Importantly, since the obtained vectors are centered at the origin, they can provide an approximation of the points $\left(\boldsymbol{G}_{i},\boldsymbol{P}_{i}\right)$ in a linear vector space, a fact that will be exploited in practice in \cref{sec:Experiments}.


\subsection{$\widetilde{\Gamma}^+$ on $\mathcal{S}^+_{d,r}$}
\label{sub:gamma_tilde_spsd}

Let $\widetilde{\Gamma}_{\overline{\boldsymbol{C}}\rightarrow\overline{\boldsymbol{Y}}}^{+}\left(\boldsymbol{C}_{i}\right)$ denote a transport of the set $\mathcal{C}$ from its mean $\overline{\boldsymbol{C}}\cong\left(\boldsymbol{\overline{\boldsymbol{G}},\overline{\boldsymbol{P}}}\right)$ to a new center of mass $\overline{\boldsymbol{Y}}\cong\left(\boldsymbol{\overline{\boldsymbol{V}},\overline{\boldsymbol{R}}}\right)$. 
Suppose $\overline{\boldsymbol{C}}\cong\left(\boldsymbol{\overline{\boldsymbol{G}},\overline{\boldsymbol{P}}}\right)$ is the canonical representation and suppose that $\overline{\boldsymbol{V}}=\Pi_{\overline{\boldsymbol{G}}}\left(\overline{\boldsymbol{V}}\right)$ and $\overline{\boldsymbol{R}}=\overline{\boldsymbol{V}}^{T}\overline{\boldsymbol{Y}}\overline{\boldsymbol{V}}$.
With the above preparation, define
\begin{align}
    \label{eq:SpsdGammaPlus}
\begin{split}
\widetilde{\Gamma}_{\overline{\boldsymbol{C}}\rightarrow\overline{\boldsymbol{Y}}}^{+}\left(\boldsymbol{C}_{i}\right) & \cong\widetilde{\Gamma}_{\left(\overline{\boldsymbol{G}},\overline{\boldsymbol{P}}\right)\rightarrow\left(\overline{\boldsymbol{V}},\overline{\boldsymbol{R}}\right)}^{+}\left(\boldsymbol{G}_{i},\boldsymbol{P}_{i}\right)\\
 &{\color{black} \coloneqq\widetilde{\text{E}}_{\left(\overline{\boldsymbol{V}},\overline{\boldsymbol{R}}\right)}\left(\widetilde{\Gamma}_{\left(\overline{\boldsymbol{G}},\overline{\boldsymbol{P}}\right)\rightarrow\left(\overline{\boldsymbol{V}},\overline{\boldsymbol{R}}\right)}\left(\widetilde{L}_{\left(\overline{\boldsymbol{G}},\overline{\boldsymbol{P}}\right)}\left(\left(\boldsymbol{G}_{i},\boldsymbol{P}_{i}\right)\right)\right)\right)}\\
 & =\left(\Gamma_{\overline{\boldsymbol{G}}\rightarrow\overline{\boldsymbol{V}}}^{+}\left(\boldsymbol{G}_{i}\right),\Gamma_{\overline{\boldsymbol{P}}\rightarrow\overline{\boldsymbol{R}}}^{+}\left(\boldsymbol{P}_{i}\right)\right)\\
 & =\left(\overline{\boldsymbol{O}}\overline{\boldsymbol{Q}}^{T}\boldsymbol{G}_{i},\boldsymbol{E}\boldsymbol{P}_{i}\boldsymbol{E}^{T}\right)\\
 & \cong\overline{\boldsymbol{O}}\overline{\boldsymbol{Q}}^{T}\boldsymbol{G}_{i}\boldsymbol{E}\boldsymbol{P}_{i}\boldsymbol{E}^{T}\boldsymbol{G}_{i}^{T}\overline{\boldsymbol{Q}}\overline{\boldsymbol{O}}^{T}
\end{split}
\end{align}
where $\boldsymbol{C}_{i}\cong\left(\boldsymbol{G}_{i},\boldsymbol{P}_{i}\right)$ is the canonical representation (see \cref{eq:Canonical}),  $\overline{\boldsymbol{Q}}$ is an orthogonal completion of $\overline{\boldsymbol{G}}$,
 $\overline{\boldsymbol{O}}$ is an orthogonal completion of $\overline{\boldsymbol{V}}$
such that $\overline{\boldsymbol{O}}=\Pi_{\overline{\boldsymbol{Q}}}\left(\overline{\boldsymbol{O}}\right)$ (see \cref{eq:GrassProjectionQ}),
and $\boldsymbol{E}=\left(\overline{\boldsymbol{R}}\overline{\boldsymbol{P}}^{-1}\right)^{\frac{1}{2}}$.

Since $M\left(\left\{ \boldsymbol{O}\boldsymbol{G}_{i}\right\} _{i}\right)=\overline{\boldsymbol{V}}$
and $M\left(\left\{ \boldsymbol{E}\boldsymbol{P}_{i}\boldsymbol{E}^{T}\right\} _{i}\right)=\overline{\boldsymbol{R}}$
(see \cref{eq:SpdMeanTransportation} and \cref{eq:GrassMeanTransportation}, respectively), the mean of the transported set $\widetilde{\Gamma}^{+}\left(\mathcal{C}\right)$
is indeed $\overline{\boldsymbol{Y}}\cong\left(\overline{\boldsymbol{V}},\overline{\boldsymbol{R}}\right)$ satisfying property 2 in \cref{def:isometricTransportation}.
\section{Domain adaptation on $\mathcal{S}^+_{d,r}$}
\label{sec:DA}


Consider two sets $\mathcal{X}=\left\{ x_{i}\in\mathcal{M}\right\} _{i=1}^{N_{x}}$ and $\mathcal{Y}=\left\{ y_{j}\in\mathcal{M}\right\} _{j=1}^{N_{y}}$, on a Riemannian manifold $\mathcal{M}$.
Suppose the two sets are intrinsically homogeneous but concentrated on different parts of the manifold, and thus, the two sets can be viewed as if they live in different domains.
To concur with the literature on DA, we will call $\mathcal{X}$ the \emph{source} domain and $\mathcal{Y}$ the \emph{target} domain. 
Multiple factors can contribute to this nuisance difference between the sets, such as different acquisition systems, sensing equipment, environments, and configurations, to name but a few.
The goal of DA is to mitigate the difference between the two sets and to provide a new representation of their union, such that any subsequent processing and analysis applied to the union could be unaware of their original partition and could treat them as one homogeneous set.

For this purpose of DA, when the data lie on $\mathcal{P}_d$, it was shown that applying $\Gamma^+$ to one set, transporting the data from its Riemannian mean to the Riemannian mean of the other set, gives rise to the desired outcome \cite{yair2019parallel}.
Analogously, when the data lie on $\mathcal{G}_{d,r}$, one could apply $\Gamma^+$ on $\mathcal{G}_{d,r}$ in a similar manner, and arguably, could expect similar performance as $\Gamma^+$ on $\mathcal{P}_d$, since the two variants of $\Gamma^+$ satisfy the two properties of \cref{def:isometricTransportation}, making them useful for DA. Namely, on $\mathcal{P}_d$ and on $\mathcal{G}_{d,r}$, the transportation is ``rigid'', that is, $\Gamma^+$ preserves pairwise distances and ``matches'' the means.
{\color{black} In fact, one could view the DA carried out by $\Gamma^{+}$ merely as the appropriate Riemannian counterpart of the a Euclidean mean subtraction. Remarkably, despite its simplicity, we will show empirically that such a DA gives rise to significantly improved results.}

Here, we propose to utilize $\widetilde{\Gamma}^+$ derived in \cref{sub:gamma_tilde_spsd} for DA on $\mathcal{S}^+_{d,r}$.
As discussed in \cref{sec:SpsdTransportation}, there is no definitive way to compute the pairwise distances on $\mathcal{S}_{d,r}^{+}$, and therefore, $\widetilde{\Gamma}^+$ is not guaranteed to admit property 1 of \cref{def:isometricTransportation}. Nevertheless, it does satisfy property 2, and, as we show in \cref{sec:Experiments}, it indeed facilitates a useful DA on the manifold of SPSD matrices.

Let $\overline{\boldsymbol{X}}$ and $\overline{\boldsymbol{Y}}$ be the Riemannian means of $\mathcal{X}=\left\{ \boldsymbol{X}_{i}\in\mathcal{S}^+_{d,r}\right\} _{i=1}^{N_{x}}$ and $\mathcal{Y}=\left\{ \boldsymbol{Y}_{j}\in\mathcal{S}^+_{d,r}\right\} _{j=1}^{N_{y}}$, respectively.
Let $\overline{\boldsymbol{X}}\cong\left(\overline{\boldsymbol{G}},\overline{\boldsymbol{P}}\right)$
and $\boldsymbol{X}_{i}\cong\left(\boldsymbol{G}_{i},\boldsymbol{P}_{i}\right)$ be the canonical representation of the mean and the SPSD matrices in $\mathcal{X}$ in the structure space as in \cref{eq:Canonical}.
In addition, let $\overline{\boldsymbol{V}}\in\mathcal{V}_{d,r}$
be the Grassmann mean of the ranges of the SPSD matrices in $\mathcal{Y}$ such that $\overline{\boldsymbol{V}}=\Pi_{\overline{\boldsymbol{G}}}\left(\overline{\boldsymbol{V}}\right)$. 
Denote
\[
\overline{\boldsymbol{R}}\coloneqq\overline{\boldsymbol{V}}^{T}\overline{\boldsymbol{Y}}\overline{\boldsymbol{V}}
\]
so that
\[
\overline{\boldsymbol{Y}}\cong\left(\overline{\boldsymbol{V}},\overline{\boldsymbol{R}}\right)
\]
These representations led in \cref{eq:SpsdGammaPlus} to the following explicit form of $\widetilde{\Gamma}_{\overline{\boldsymbol{X}}\rightarrow\overline{\boldsymbol{Y}}}^{+}$
\[
\widetilde{\Gamma}_{\overline{\boldsymbol{X}}\rightarrow\overline{\boldsymbol{Y}}}^{+}\left(\boldsymbol{X}_{i}\right)=\overline{\boldsymbol{O}}\overline{\boldsymbol{Q}}^{T}\boldsymbol{G}_{i}\boldsymbol{E}\boldsymbol{P}_{i}\boldsymbol{E}^{T}\boldsymbol{G}_{i}^{T}\overline{\boldsymbol{Q}}\overline{\boldsymbol{O}}^{T}
\]

With the above preparation, the proposed DA algorithm culminates in the application of $\widetilde{\Gamma}_{\overline{\boldsymbol{X}}\rightarrow\overline{\boldsymbol{Y}}}^{+}$ to every SPSD matrix $\boldsymbol{X}_i$ in $\mathcal{X}$, obtaining a new representation $\widetilde{\boldsymbol{X}}_i$. The complete DA algorithm is given in \cref{alg:SPSD_DA} and the source code is available in\footnote{\href{https://github.com/oryair/Symmetric-Positive-Semi-definite-Riemannian-Geometry-with-Application-to-Domain-Adaptation}{SpsdDomainAdaptation GitHub}}.


The proposed algorithm has several important features for DA. (i) The algorithm does not require many data points, because the computation of $\widetilde{\Gamma}^+$ only depends on coarse estimates of the two means $\overline{\boldsymbol{X}}=M\left(\mathcal{X}\right)$ and $\overline{\boldsymbol{Y}}=M\left(\mathcal{Y}\right)$.
(ii) Once $\widetilde{\Gamma}^+$ is computed, it can be applied to new unseen data. Let $\boldsymbol{X}^{\star}$ be a new (unseen) point obtained from the source domain, then $\widetilde{\boldsymbol{X}}^{\star}=\widetilde{\Gamma}_{\overline{\boldsymbol{X}}\to\overline{\boldsymbol{Y}}}^{+}\left(\boldsymbol{X}^{\star}\right)$ is the corresponding out of sample extension.
(iii) The extension of the algorithm to multiple data sets is straight-forward.
When more than two sets are given, one set can be designated as a reference (target) set, and then, all the remaining (source) sets are transported (one-by-one) to that reference set.
Recall that $\widetilde{\Gamma}^{+}$ is completely unsupervised, namely, no labels are required. Hence, any set can be chosen as the reference set.

One shortcoming of the algorithm is that it makes use only of the first order statistics. Namely, if two source sets $\mathcal{X}_{1}$ and $\mathcal{X}_{2}$
have different high order statistics but share the
same mean $\overline{\boldsymbol{X}}_{1}=\overline{\boldsymbol{X}}_{2}$, the transportations $\widetilde{\Gamma}_{\overline{\boldsymbol{X}}_1\rightarrow\overline{\boldsymbol{Y}}}^{+}$ and $\widetilde{\Gamma}_{\overline{\boldsymbol{X}}_2\rightarrow\overline{\boldsymbol{Y}}}^{+}$ (to $\overline{\boldsymbol{Y}}$) are identical.
For large data sets, where higher order statistics can be accurately estimated, 
we outline two possible modifications.
First, based on \cite{maman2019domain}, the proposed algorithm can be supplemented with a second moments alignment step.
Specifically, recall that $\widetilde{\Gamma}^{+}$ is a composition of three steps: (i) projection to the tangent space, (ii) application of PT, and (iii) projection back to the manifold.
Let $\left\{ \left(\widetilde{\boldsymbol{\Delta}}_{i},\widetilde{\boldsymbol{S}}_{i}\right)\in\mathcal{T}_{\left(\overline{\boldsymbol{V}},\overline{\boldsymbol{R}}\right)}\mathcal{S}_{d,r}^{+}\right\} _{i=1}^{N_{x}}$ be the tangent vectors obtained after step (ii).
Instead of projecting back to the manifold at step (iii), we propose to project the target set $\mathcal{Y}$ to $\mathcal{T}_{\left(\overline{\boldsymbol{V}},\overline{\boldsymbol{R}}\right)}\mathcal{S}_{d,r}^{+}$ as well. Now, the sets (the transported source set and the projected target set) are points in a vector space. This allow us to rotate the source set $\left\{ \left(\widetilde{\boldsymbol{\Delta}}_{i},\widetilde{\boldsymbol{S}}_{i}\right)\in\mathcal{T}_{\left(\overline{\boldsymbol{V}},\overline{\boldsymbol{R}}\right)}\mathcal{S}_{d,r}^{+}\right\} _{i=1}^{N_{x}}$ such that its second moments are aligned with the second moments of the target set.
In order to overcome the ambiguity in the orientation of the rotation, in
\cite{maman2019domain}, it was proposed to rotate each axis according to the smaller angle.

The second modification is based on \cite{courty2016optimal} and \cite{yair2019optimal}, where DA is carried out by solving a regularized optimal transport problem \cite{courty2014domain}.
There, the cost is based on the (squared) length of the curve $\widetilde{\gamma}_{\boldsymbol{X}_{i}\to\boldsymbol{Y}_{j}}$, and the transportation is applied using a weighted mean. Since the length of the curve $\widetilde{\gamma}$ is not a metric and the weighted mean on $\mathcal{S}_{d,r}^{+}$ needs to be developed, we postpone the development of this transportation to future work. We note that such a transformation does not aim to preserve pairwise distances, but rather, to align the respective distributions of the two sets $\mathcal{X}$ and $\mathcal{Y}$.


\begin{algorithm}
\caption{Domain Adaptation on $\mathcal{S}_{d,r}^{+}$}
    \label{alg:SPSD_DA}

\textbf{\uline{Input}}\textbf{:} Two sets of SPSD matrices $\mathcal{X}=\left\{ \boldsymbol{X}_{i}\in\mathcal{S}_{d,r}^{+}\right\} _{i=1}^{N_{x}}$
and $\mathcal{Y}=\left\{ \boldsymbol{Y}_{i}\in\mathcal{S}_{d,r}^{+}\right\} _{i=1}^{N_{y}}$.

\textbf{\uline{Output}}\textbf{:} The set $\widetilde{\mathcal{X}}=\left\{ \widetilde{\boldsymbol{X}}_{i}\in\mathcal{S}_{d,r}^{+}\right\} _{i=1}^{N_{x}}$ adapted to the domain of $\mathcal{Y}$.
\begin{enumerate}
\item Obtain the means of $\mathcal{X}$ and $\mathcal{Y}$:
\begin{enumerate}
\item Compute the SPSD and Grassmann means of $\mathcal{X}$ and $\mathcal{Y}$
\hfill $\rhd$ using \cref{alg:SpsdMean}
\begin{enumerate}
\item $\overline{\boldsymbol{X}}\in\mathcal{S}_{d,r}^{+}$ and $\overline{\boldsymbol{G}}\in\mathcal{V}_{d,r}$
for $\mathcal{X}$
\item $\overline{\boldsymbol{Y}}\in\mathcal{S}_{d,r}^{+}$ and $\overline{\boldsymbol{V}}\in\mathcal{V}_{d,r}$
for $\mathcal{Y}$
\end{enumerate}
\item Set $\overline{\boldsymbol{V}}\leftarrow\Pi_{\overline{\boldsymbol{G}}}\left(\overline{\boldsymbol{V}}\right)$ 
\hfill $\rhd$ using \cref{eq:GrassProjectionG}
\item Set $\overline{\boldsymbol{P}}=\overline{\boldsymbol{G}}^{T}\overline{\boldsymbol{X}}\overline{\boldsymbol{G}}$
and $\overline{\boldsymbol{R}}=\overline{\boldsymbol{V}}^{T}\overline{\boldsymbol{Y}}\overline{\boldsymbol{V}}$
\end{enumerate}

\item For $i=1,2\dots,N_{x}$:
\begin{enumerate}
\item Compute the canonical representation: $\boldsymbol{X}_{i}\cong\left(\boldsymbol{G}_{i},\boldsymbol{P}_{i}\right)$ such that $\boldsymbol{G}_{i}=\Pi_{\overline{\boldsymbol{G}}}\left(\boldsymbol{G}_{i}\right)$
$\rhd$ see \cref{eq:Canonical}

\item Compute ($\rhd$ see \cref{eq:SpsdGammaPlus}):
\[
\left(\widetilde{\boldsymbol{G}}_{i},\widetilde{\boldsymbol{P}}_{i}\right)=\widetilde{\Gamma}_{\left(\overline{\boldsymbol{G}},\overline{\boldsymbol{P}}\right)\rightarrow\left(\overline{\boldsymbol{V}},\overline{\boldsymbol{R}}\right)}^{+}\left(\boldsymbol{G}_{i},\boldsymbol{P}_{i}\right)=\left(\overline{\boldsymbol{O}}\overline{\boldsymbol{Q}}^{T}\boldsymbol{G}_{i},\boldsymbol{E}\boldsymbol{P}_{i}\boldsymbol{E}^{T}\right)
\]
where $\overline{\boldsymbol{Q}}$ is an orthogonal completion of $\overline{\boldsymbol{G}}$, $\overline{\boldsymbol{O}}$ is an orthogonal completion of $\overline{\boldsymbol{V}}$ such that $\overline{\boldsymbol{O}}=\Pi_{\overline{\boldsymbol{Q}}}\left(\overline{\boldsymbol{O}}\right)$ (see \cref{eq:GrassProjectionQ}),
and $\boldsymbol{E}=\left(\overline{\boldsymbol{R}}\overline{\boldsymbol{P}}^{-1}\right)^{\frac{1}{2}}$.
\item Set:
\[
\widetilde{\boldsymbol{X}}_{i}=\widetilde{\boldsymbol{G}}_{i}\widetilde{\boldsymbol{P}}_{i}\widetilde{\boldsymbol{G}}_{i}^{T}
\]
\end{enumerate}
\end{enumerate}
\end{algorithm}

\section{Experimental study}
    \label{sec:Experiments}
\subsection{Hyper-spectral imaging}
    \label{sub:HSI}
To demonstrate our algorithm for DA we apply it to a real hyper-spectral dataset. Hyper-Spectral Imaging (HSI) measures multiple spectral bands of the light reflected from a spatial area. Recent technological advances allow for the acquisition of hundreds of spectral bands which encode rich information on the captured scene. Therefore, a large and growing number of studies have addressed the challenge of analyzing and processing hyper-spectral images for various purposes, such as classification \cite{deng2018active,fang2018new,gross2019nonlinear}, change detection \cite{wu2017kernel} and target detection \cite{kang2017hyperspectral}.

A well known problem in analyzing two or more hyper-spectral images is 
the inherent diversity between different images.
This diversity could be the result of differences in illumination, viewing angle, sensor configuration, and even the type of sensors. 
In order to analyze two different images, or to exploit the model learned from
one image, say, $\boldsymbol{I}^{(1)}$, for analysis tasks in another image, say, $\boldsymbol{I}^{(2)}$, DA is required. 
We next explain how we apply our approach to this purpose. 

Consider a hyper-spectral image organized in a $3D$ cube $\boldsymbol{I} \in \mathbb{R}^{n_x \times n_y \times n_b}$, which is also referred as the hyper-spectral cube, where $n_x$ and $n_y$ are the spatial dimensions and $n_b$ is the number of spectral bands.
Let $\boldsymbol{p}_{i} \in \mathbb{R}^{n_b}$, $i=1,2,\ldots,n_x n_y$, be the $i_{th}$ pixel of $\boldsymbol{I}$ representing a local spectral signature. 
Recently, it was shown in \cite{fang2018new} that a good spatial-spectral feature of $\boldsymbol{p}_{i}$, which expresses the relations between the spectral bands, is the local covariance matrix $\boldsymbol{X}_{i} \in \mathbb{R}^{n_b \times n_b}$, given by
\begin{equation}
\boldsymbol{X}_{i}=\frac{1}{\left|\mathcal{N}_{i} \right|-1}\sum_{\boldsymbol{p}_{j}\in\mathcal{N}_{i}}\left(\boldsymbol{p}_{j}-\boldsymbol{\mu}_{i}\right)\left(\boldsymbol{p}_{j}-\boldsymbol{\mu}_{i}\right)^{T}
\label{eq:sample_cov}
\end{equation}
where $\mathcal{N}_{i}$ are the $J$ nearest neighbors of $\boldsymbol{p}_{i}$ from all the pixels in a patch of size $W\times W$ centered at $\boldsymbol{p}_i$, and $\boldsymbol{\mu}_i = \frac{1}{\left|\mathcal{N}_{i} \right|}\sum_{\boldsymbol{p}_{j}\in\mathcal{N}_{i}}\boldsymbol{p}_j$. The nearest neighbors are chosen with respect to the angular (cosine) similarity
$$
\theta_{ij}=\arccos\left(\frac{\boldsymbol{p}_{i}^{T}\boldsymbol{p}_{j}}{\left\Vert \boldsymbol{p}_{i}\right\Vert _{2}\left\Vert \boldsymbol{p}_{j}\right\Vert _{2}}\right)
$$
A common assumption in HSI is that the spectral signature $\boldsymbol{p}_i$ is a linear combination of a small number $r < n_b$ of spectral profiles \cite{zhang2013hyperspectral,iordache2011sparse}.
According to this assumption, $\boldsymbol{X}_{i}$ is an SPSD matrix with rank $r$.
Therefore, we can use our approach to adapt two (or more) hyper-spectral images $\boldsymbol{I}^{\left(1\right)}\in\mathbb{R}^{n_{x}^{\left(1\right)}\times n_{y}^{\left(1\right)}\times n_{b}}$ and $\boldsymbol{I}^{\left(2\right)}\in\mathbb{R}^{n_{x}^{\left(2\right)}\times n_{y}^{\left(2\right)}\times n_{b}}$ as follows:
(i) Compute the local covariance matrices $\mathcal{X}^{(1)}=\{\boldsymbol{X}^{(1)}_{i}  \in \mathcal{S}_{n_b,r}^{+}\}_{i=1}^{n_x^{(1)} n_y^{(1)}}$, of pixels from $\boldsymbol{I}^{(1)}$ and the local covariance matrices $\mathcal{X}^{(2)}=\{\boldsymbol{X}^{(2)}_{i}  \in \mathcal{S}_{n_b,r}^{+}\}_{i=1}^{n_x^{(2)} n_y^{(2)}}$, of pixels from $\boldsymbol{I}^{(2)}$. (ii) Transport $\mathcal{X}^{\left(1\right)}$ to the domain of $\mathcal{X}^{\left(2\right)}$ by applying \cref{alg:SPSD_DA} (giving rise to $\widetilde{\mathcal{X}}^{\left(1\right)}$).

We apply our method for the purpose of adapting hyper-spectral images of the same scene but with different time of acquisition, taken from the Greding dataset \cite{gross2019nonlinear}.
After removing rows and columns with non-valid pixels, the dimensions of the images are $n_x=626$, $n_y=591$ and $n_b = 127$. \cref{fig:Greding_RGB} shows an RGB representation (3 channels) of two images from the Greding dataset: \texttt{Greding\_Village1\_refl}, denoted by $\boldsymbol{I}^{(1)}$, and \texttt{Greding\_Village3\_refl}, denoted by $\boldsymbol{I}^{(2)}$. It can be visually observed that at least the illumination in these two images is different.

For the local covariance computation in \eqref{eq:sample_cov}, we use the same parameters as in \cite{fang2018new}: patch size $W=25$, number of neighbor pixels $J=220 $, and we set the rank to be $r=40$ because empirically it attains good performance. A similar rank was reported in \cite{fang2018new}. To reduce the computational load of \cref{alg:SPSD_DA}, we use subsets of 500 matrices $\mathcal{X}_s^{(1)} \subset \mathcal{X}^{(1)}$ and $\mathcal{X}_s^{(2)} \subset \mathcal{X}^{(2)}$ chosen randomly for the  mean computation in step 1(a), instead of the entire sets. 

In order to use a feasible amount of memory and to obtain a representation in a linear space, after applying \cref{alg:SPSD_DA} we represent each SPSD matrix only by its $40$ Principal Components (PCs) computed as follows. 
First, we compute the logarithmic map approximation \eqref{eq:SpsdLogMap} for each $ \boldsymbol{X}_i \in  \mathcal{X}^{(1)} \cup \mathcal{X}^{(2)}$, and get the corresponding vector in the tangent space $\left(\boldsymbol{\Delta}_i,\boldsymbol{S}_i\right) \in \mathcal{T}_{\overline{\boldsymbol{X}}}\mathcal{S}_{n_b,r}$, where $\overline{\boldsymbol{X}}$ is the mean of $\mathcal{X}_s^{(1)} \cup \mathcal{X}_s^{(2)}$. Second, we represent each tangent vector $\left(\boldsymbol{\Delta}_i,\boldsymbol{S}_i\right)$ by a column stack of $\boldsymbol{\Delta}_i$ and $k \boldsymbol{S}_i$, 
denoted by $\xi_i$,
where $k$ was set such that the (global) standard deviation of $\left\{ \boldsymbol{\Delta}_{i}\right\} _{i}$ is the same as $\left\{ k\boldsymbol{S}_{i}\right\} _{i}$. Third, we compute the $40$ principal directions of the vectors in the tangent space by (i) applying SVD to the matrix whose columns are $\xi_i$ and (ii) taking the left singular vectors corresponding to the largest singular values. Fourth, we project each vector $\xi_i$ on the obtained $40$ principal directions, getting a new vector in $\mathbb{R}^{40}$ consisting of the $40$ PCs.

\begin{figure}
\centering
\subcaptionbox{}
{\includegraphics[width=0.35\paperwidth]{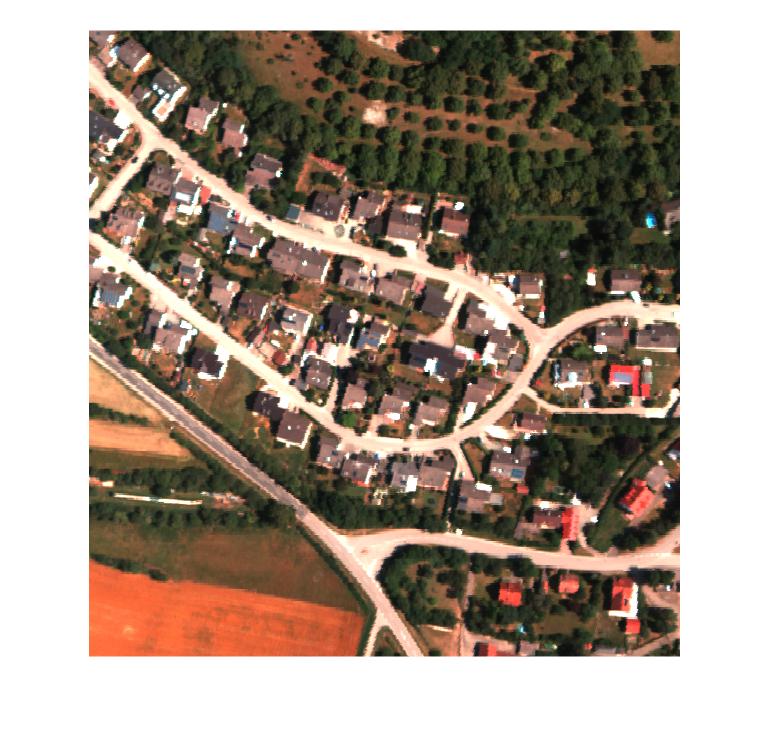}}
\subcaptionbox{}
{\includegraphics[width=0.35\paperwidth]{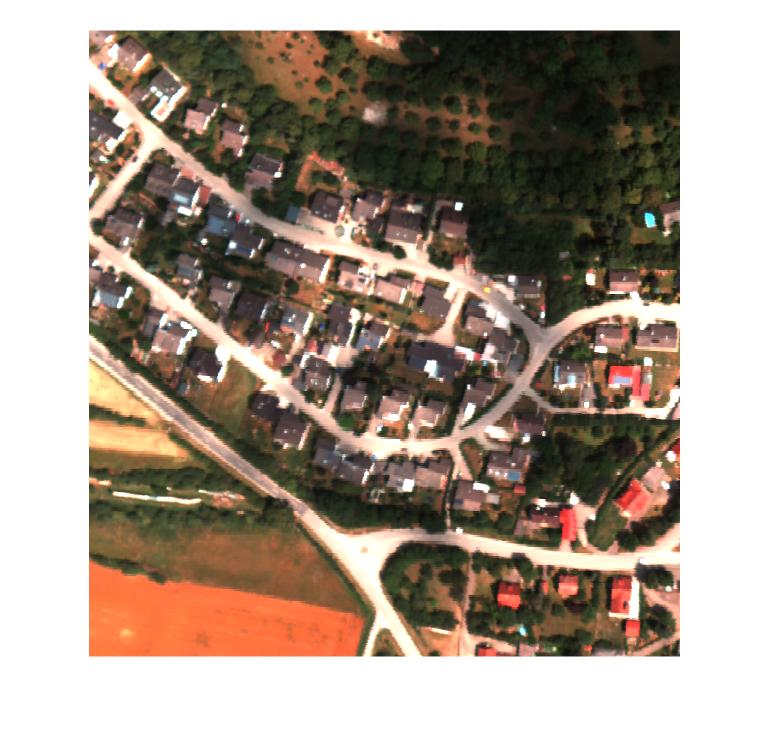}}
\caption{RGB representation (3 channels) of the two hyperspectral images $\boldsymbol{I}^{(1)}$ and $\boldsymbol{I}^{(2)}$, from the Greding dataset \cite{gross2019nonlinear}.
}
\label{fig:Greding_RGB}
\end{figure}

The obtained DA is evaluated by using $6$ land-cover labels from the Greding dataset: {\tt dark roof}, {\tt red roof}, {\tt concrete}, {\tt soil}, {\tt grass}, and {\tt tree}. We denote the set of local covariance matrices in $\boldsymbol{I}^{(1)}$ after applying \cref{alg:SPSD_DA} by $\mathcal{\widetilde{X}}^{(1)}$. \cref{fig:Greding_scat}(a) and \cref{fig:Greding_scat}(b) show the two PCs of $400$ matrices from $\mathcal{{X}}^{(1)}\cup\mathcal{X}^{(2)}$ and $\mathcal{\widetilde{X}}^{(1)}\cup\mathcal{X}^{(2)}$ respectively, where points are colored according to the land-cover labels. As can be seen in \cref{fig:Greding_scat}(a), points from $\mathcal{{X}}^{(1)}$, marked by circles, and points from $\mathcal{{X}}^{(2)}$, marked by asterisks, with the same label (color) reside in different regions. Conversely, in \cref{fig:Greding_scat}(b), after applying \cref{alg:SPSD_DA}, we observe that points both from $\mathcal{\widetilde{X}}^{(1)}$ and from $\mathcal{X}^{(2)}$ with the same label lie at the same region.

To evaluate our method numerically, we repeat the experiment reported in \cite{gross2019nonlinear}.
We train a linear SVM classifier on 10\% of the $40$ PCs of the local covariance matrices in $\mathcal{\widetilde{X}}^{(1)}$ and test it on the $40$ PCs of the local covariance matrices in $\mathcal{{X}}^{(2)}$. 
We remark that pixels at the boundaries and pixels without enough valid neighbors for the covariance estimation are ignored. We consider only pixels with at least 500 valid neighbors, which applies to 95\% of the labeled pixels.
\cref{fig:Greding_im_class} illustrates the classification results in the image plain, where pixels are colored according to their predicted class.
We use the following Cohen's kappa \cite{cohen1960coefficient} to objectively evaluate the classification results
$$
\kappa = \frac{p_o - p_e}{1 - p_e}
$$
where $p_o$ is the classification accuracy and $p_e$ is given by
$$
p_e = \frac{1}{N^2}\sum_k n^{(T)}_k n^{(P)}_k
$$
where $N$ is the number of observations to be classified, $n^{(T)}_k$ and $n^{(P)}_k$ is the true and predicted number of observations in class $k$, respectively.
After applying  \cref{alg:SPSD_DA}, the obtained Cohen's kappa  of the SVM classifier is $\kappa= 0.957 \pm 0.006$, where $0.957$ is the mean $\kappa$ over $10$ repetitions and $0.006$ is the standard deviation, while without DA it is only $\kappa = 0.575 \pm 0.041$. 
\cref{kappa_tab} compares our results to other DA algorithms reported in \cite{gross2019nonlinear}. We note that according to the reported setting in \cite{gross2019nonlinear}, the STCA, KEMA, GFK and NFNalign algorithms used the reflectance mode for image $\boldsymbol{I}^{(1)}$ and the radiance mode for $\boldsymbol{I}^{(2)}$, while the re-normalization and our algorithm used the reflectance mode for both images.

\begin{table}
\centering
\caption{Comparison of the classification results after applying different DA algorithms. The column ``Unsupervised'' indicates that the algorithm does not require labels. The column ``Unpaired'' indicates that the algorithm is not restricted to images which are defined on a common grid. The column ``Generic'' indicates that the algorithm could be used for different datasets and is not specifically-tailored for HSI.   \label{kappa_tab}}
\begin{tabular}{ p{3cm}|P{2.4cm}|P{2cm}|P{2.4cm}|>{\raggedright\arraybackslash}P{2.4cm} }
 Algorithm & Unsupervised & Unpaired & Generic & $\kappa $ (SVM)\\
 \hline
 NFNalign          &         & $\surd$ &         & 0.975 \\
 re-normalization  & $\surd$ &         & $\surd$ & 0.942 \\
 STCA              &         & $\surd$ &         & 0.901 \\
 KEMA              &         & $\surd$ & $\surd$ & 0.932 \\ 
 GFK               & $\surd$ & $\surd$ & $\surd$ & 0.920 \\
 \textbf{Proposed} & $\surd$ & $\surd$ & $\surd$ & 0.957 $\pm$ 0.006 \\
 \hline
\end{tabular}
\end{table}

\begin{figure}
\centering
\subcaptionbox{}
{\includegraphics[width=0.35\paperwidth]{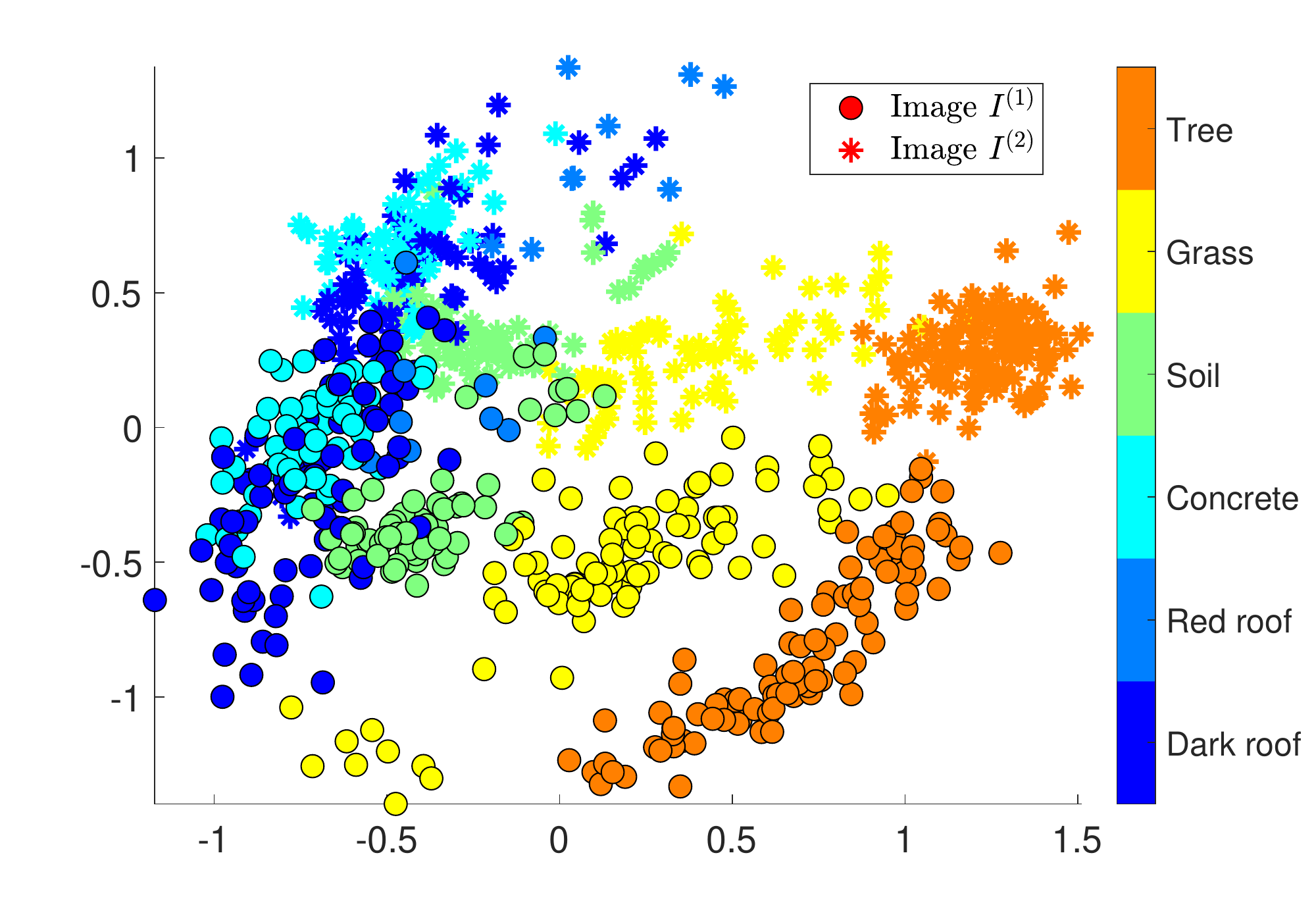}}
\subcaptionbox{}
{\includegraphics[width=0.35\paperwidth]{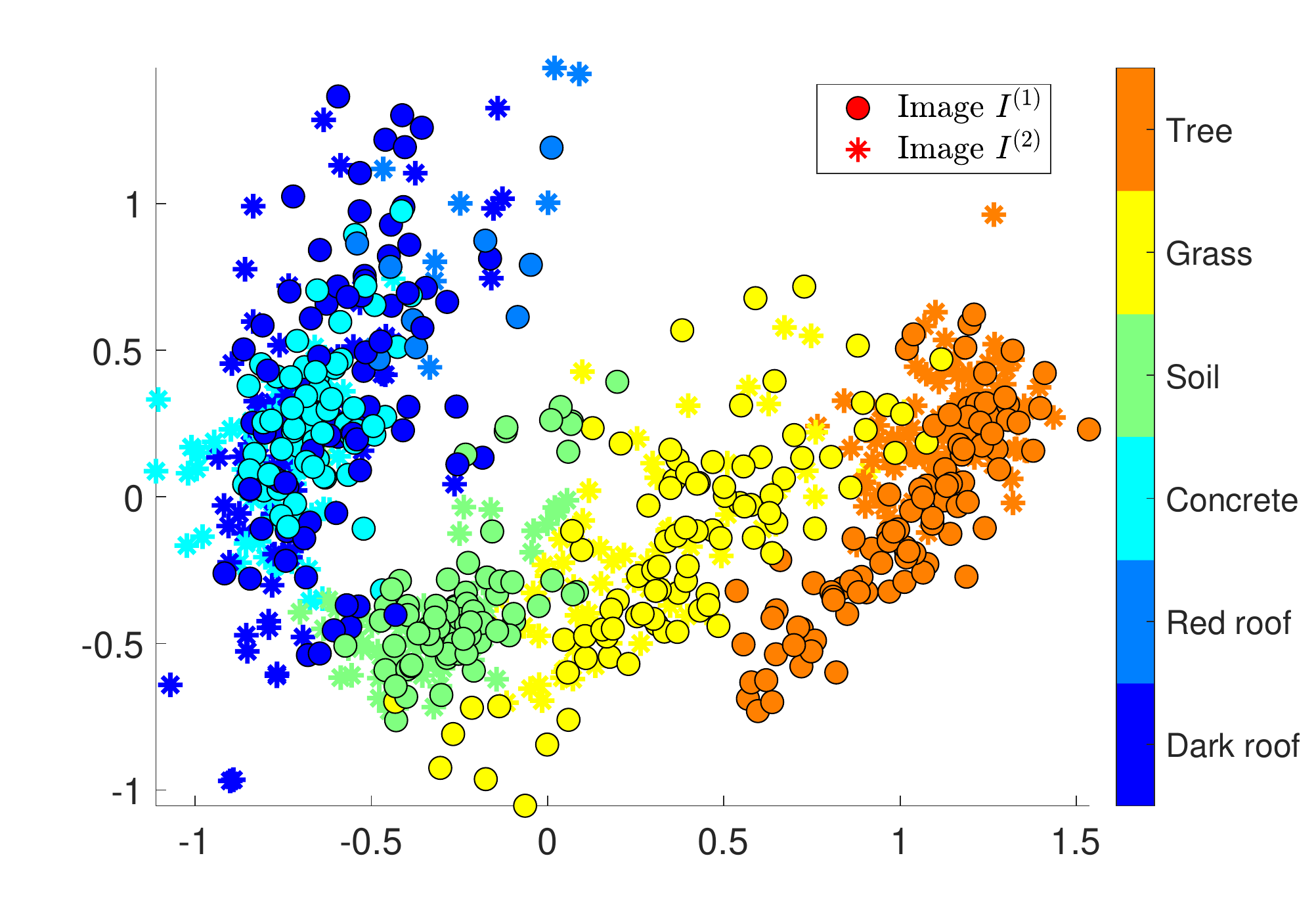}}
\caption{First 2 PC of 400 covariance matrices. Matrices computed in image $\boldsymbol{I}^{(1)}$ marked by circles, and matrices computed in image $\boldsymbol{I}^{(2)}$ marked by asterisks: (a) before DA, (b) after DA using \cref{alg:SPSD_DA} .Points are colored according to the land-cover classes.}\label{fig:Greding_scat}
\end{figure}
\begin{figure}
\centering
\subcaptionbox{}
{\includegraphics[width=0.233\paperwidth]{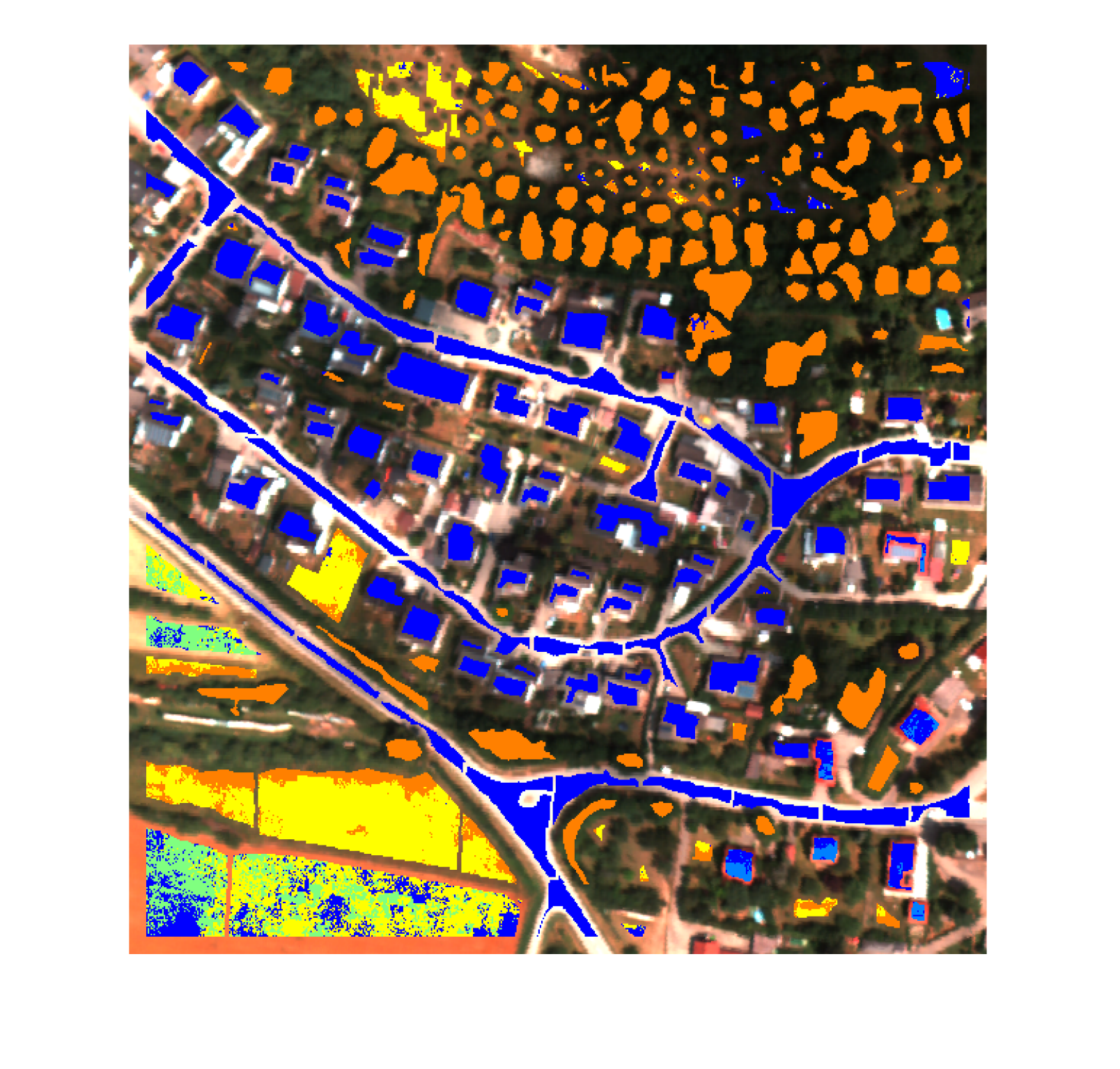}}
\subcaptionbox{}
{\includegraphics[width=0.233\paperwidth]{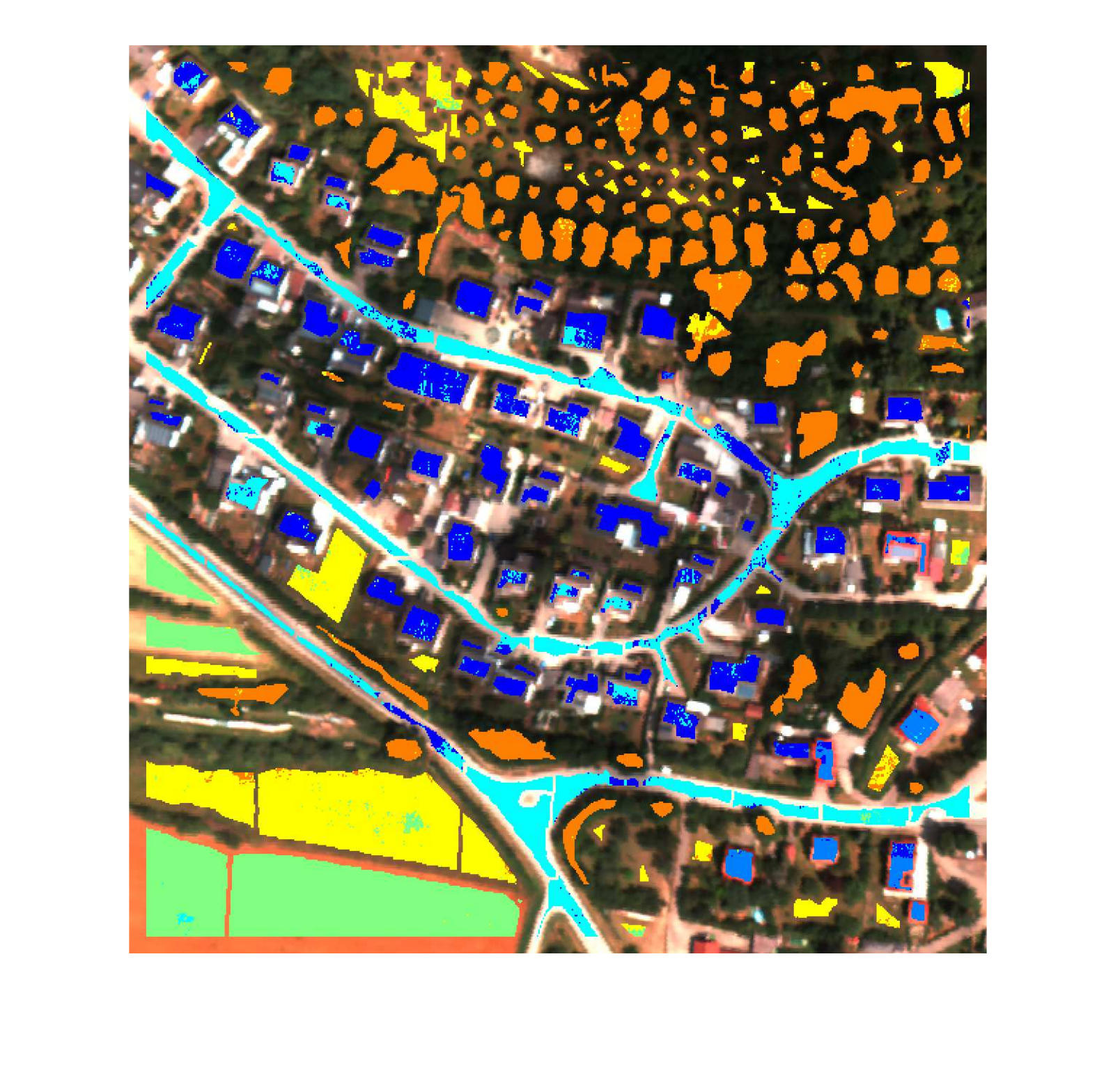}}
\subcaptionbox{}
{\includegraphics[width=0.233\paperwidth]{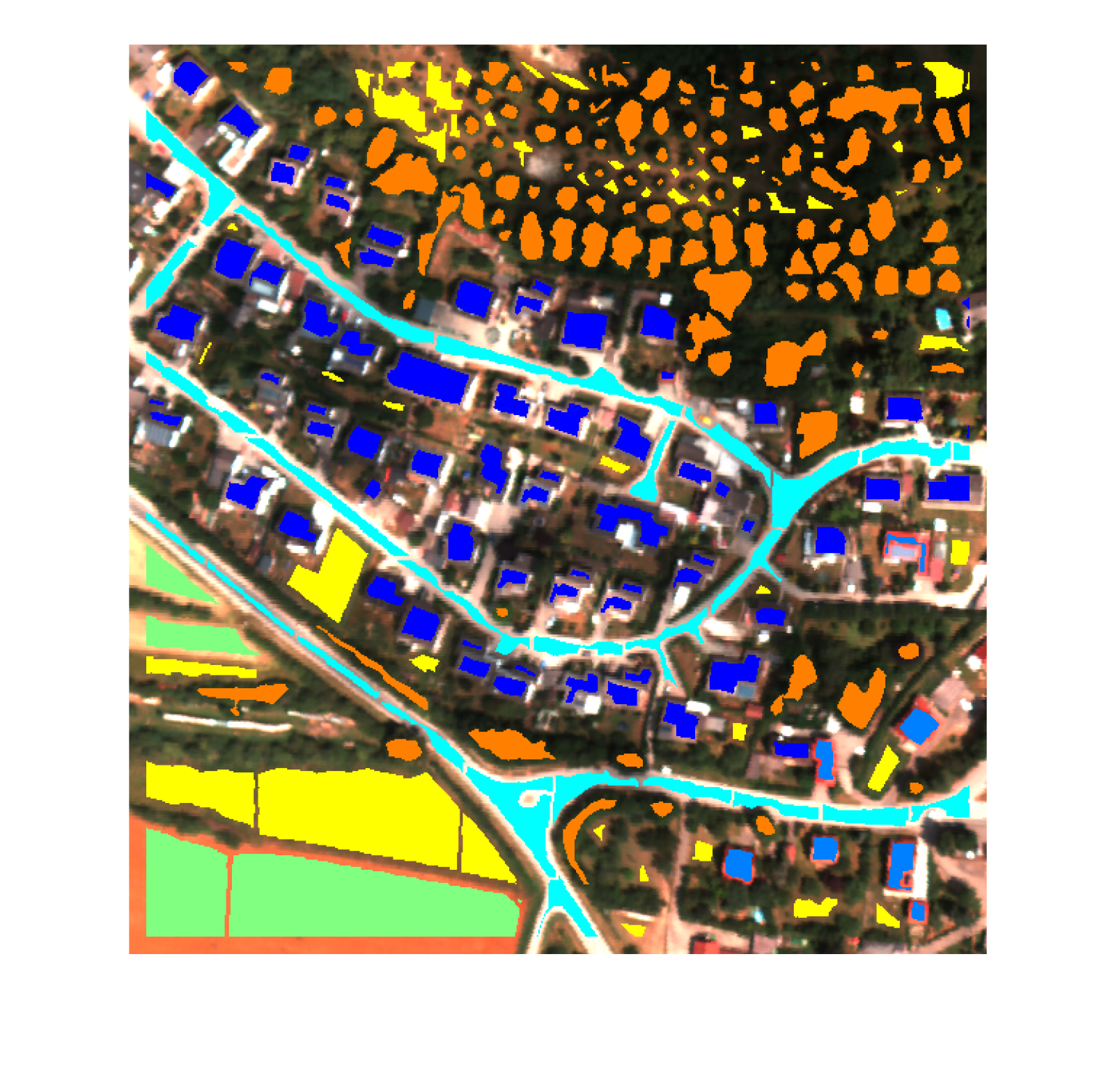}}
\caption{Classification results. Labeled pixels are colored according to the predicted land-cover classes: (a) before DA, (b) after DA using \cref{alg:SPSD_DA} and (c) ground truth.}\label{fig:Greding_im_class}
\end{figure}

\subsection{Motion recognition}
In this experiment, we use the proposed method to apply domain adaption to the motion capture database HDM05 which is described in \cite{cg-2007-2}.
The HDM05 dataset contains more than 70 motion classes in 10 to 50 realizations executed by various actors. 
Some motions for example are: a cartwheel (left hand start), a clap (1 repetition), and a clap above head (1 repetition).
The dataset contains recordings from five different actors that we denote for simplicity by Actor \#1 to Actor \#5.
The data are acquired from 31 markers that are attached to the actor's body throughout the motion, see \cref{fig:HDM05}(a).
Specifically, each marker provides a 3D position at each time frame, see \cref{fig:HDM05}(b).
A single motion is about $3$ second long recorded at $120$Hz sampling rate. Overall, we write the $i$th motion of the $k$th actor as 
$$
\boldsymbol{m}_{i}^{\left(k\right)}=\mathbb{R}^{31\times3\times T_{i}^{\left(k\right)}}
$$
where the first dimension represents a specific marker, the second dimension represents the $x,y,z$ coordinates, and $T_{i}^{\left(k\right)}\approx360$ is the number of frames in the motion.
From each motion $\boldsymbol{m}_{i}^{\left(k\right)}$, we compute the $\boldsymbol{X}_{i}^{\left(k\right)}\in\mathbb{R}^{93\times93}$ covariance matrix (by flattening the first two dimensions of $\boldsymbol{m}_{i}^{\left(k\right)}$ into a column stack vector).
Empirically, we found that only four eigenvalues are consistently greater than zero for most $\boldsymbol{X}_{i}^{\left(k\right)}$. Thus, we set the fixed rank to $r=4$, and as a consequence, we view the covariance matrices as points on $\mathcal{S}_{93,4}^{+}$.

\begin{figure}
    \centering
    \includegraphics[width=0.8\columnwidth]{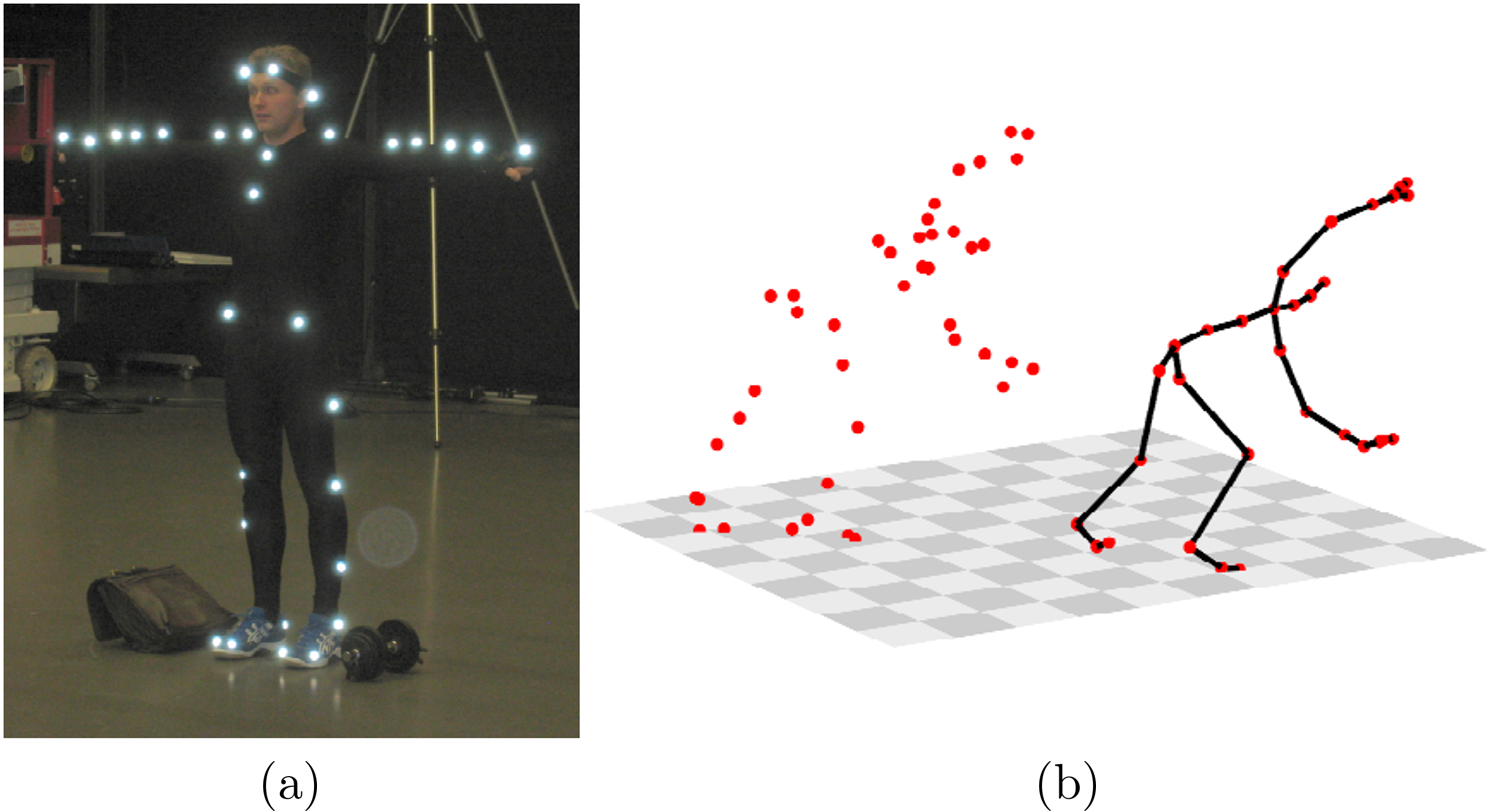}
    \caption{(a) The 31 markers attached to the actor's body.
             (b) A frame from a cartwheel motion.
             Both images were taken from \cite{cg-2007-2}.}
    \label{fig:HDM05}
\end{figure}

To demonstrate the need for domain adaption between different actors, we take Actors \#1 and \#3 and consider all motions with more than $20$ repetitions (combined). This provides us with the two sets $\left\{ \boldsymbol{X}_{i}^{\left(1\right)}\right\} _{i=1}^{48}$ and $\left\{ \boldsymbol{X}_{i}^{\left(3\right)}\right\} _{i=1}^{42}$.
\cref{fig:Sub1and3}(a) presents the 2D representation obtained by projecting the union $\left\{ \boldsymbol{X}_{i}^{\left(1\right)}\right\} _{i=1}^{48}\cup\left\{ \boldsymbol{X}_{i}^{\left(3\right)}\right\} _{i=1}^{42}$ to the tangent plane $\mathcal{T}_{\overline{\boldsymbol{C}}}\mathcal{S}_{d,r}^{+}$ where $\overline{\boldsymbol{C}}$ is the mean of the union using \cref{eq:SpsdLogMap}.
Next, we apply tSNE \cite{maaten2008visualizing} to the obtained vectors and using the induced metric by the inner product \cref{eq:SPSD_inner} with $k=1$.
Motions corresponding to Actor \#1 are marked by circles (with black edges) and motions corresponding to Actor \#3 are marked by asterisks.
Different colors correspond to different motion types.
We observe that the same motions by different actors do not reside in the same vicinity.
To circumvent this undesired discrepancy we apply \cref{alg:SPSD_DA} to $\left\{ \boldsymbol{X}_{i}^{\left(1\right)}\right\} _{i=1}^{48}$ and $\left\{ \boldsymbol{X}_{i}^{\left(3\right)}\right\} _{i=1}^{42}$ and obtain the new SPSD representation $\left\{ \widetilde{\boldsymbol{X}}_{i}^{\left(1\right)}\right\} _{i=1}^{48}$.
\cref{fig:Sub1and3}(b) presents the 2D representation obtained by applying tSNE to the union $\left\{ \widetilde{\boldsymbol{X}}_{i}^{\left(1\right)}\right\} _{i=1}^{48}\cup\left\{ \boldsymbol{X}_{i}^{\left(3\right)}\right\} _{i=1}^{42}$.
We now observe that the same type of motions recorded from different actors reside in the same vicinity, thereby implying that we have achieved a meaningful DA between the two actors.
\cref{fig:Sub2and3} is similar to \cref{fig:Sub1and3} but with Actors \#2 and \#3 (instead of \#1 and \#3).
To provide quantitative results, we train linear SVM classifiers using the SPSD matrices of each actor and test the classification accuracy on all other actors. 
Specifically, {\color{black} as in \cref{sub:HSI}}, given the two sets of SPSD matrices $\left\{ \boldsymbol{X}_{i}^{\left(k_{1}\right)}\right\} _{i}$ and $\left\{ \boldsymbol{X}_{i}^{\left(k_{2}\right)}\right\} _{i}$, the classifiers were trained in the tangent space $\mathcal{T}_{\overline{\boldsymbol{X}}}\mathcal{S}_{93,4}$ where $\overline{\boldsymbol{X}}$ is the mean of the union $\overline{\boldsymbol{X}}=M\left(\left\{ \boldsymbol{X}_{i}^{\left(k_{1}\right)}\right\} \cup\left\{ \boldsymbol{X}_{i}^{\left(k_{2}\right)}\right\} \right)$.
We repeat this experiment twice, once before applying \cref{alg:SPSD_DA} that is, we use $\left\{ \boldsymbol{X}_{i}^{\left(k_{1}\right)}\right\} _{i}$ and $\left\{ \boldsymbol{X}_{i}^{\left(k_{2}\right)}\right\} _{i}$, and once after applying \cref{alg:SPSD_DA}, that is, we use $\left\{ \widetilde{\boldsymbol{X}}_{i}^{\left(k_{1}\right)}\right\} _{i}$ and $\left\{ \boldsymbol{X}_{i}^{\left(k_{2}\right)}\right\} _{i}$. We note that we omit Actor \#4 since the number of common motions between this actor and all other actors is too small.
\cref{tab:HDM05accuracy}(a) presents the classification accuracy obtained before applying \cref{alg:SPSD_DA}. 
\cref{tab:HDM05accuracy}(b) presents the classification accuracy obtained after applying \cref{alg:SPSD_DA}.
We observe that in all cases (except one) applying \cref{alg:SPSD_DA} indeed improve the classification accuracy significantly.

\begin{figure}
    \centering
    \includegraphics[width=0.8\columnwidth]{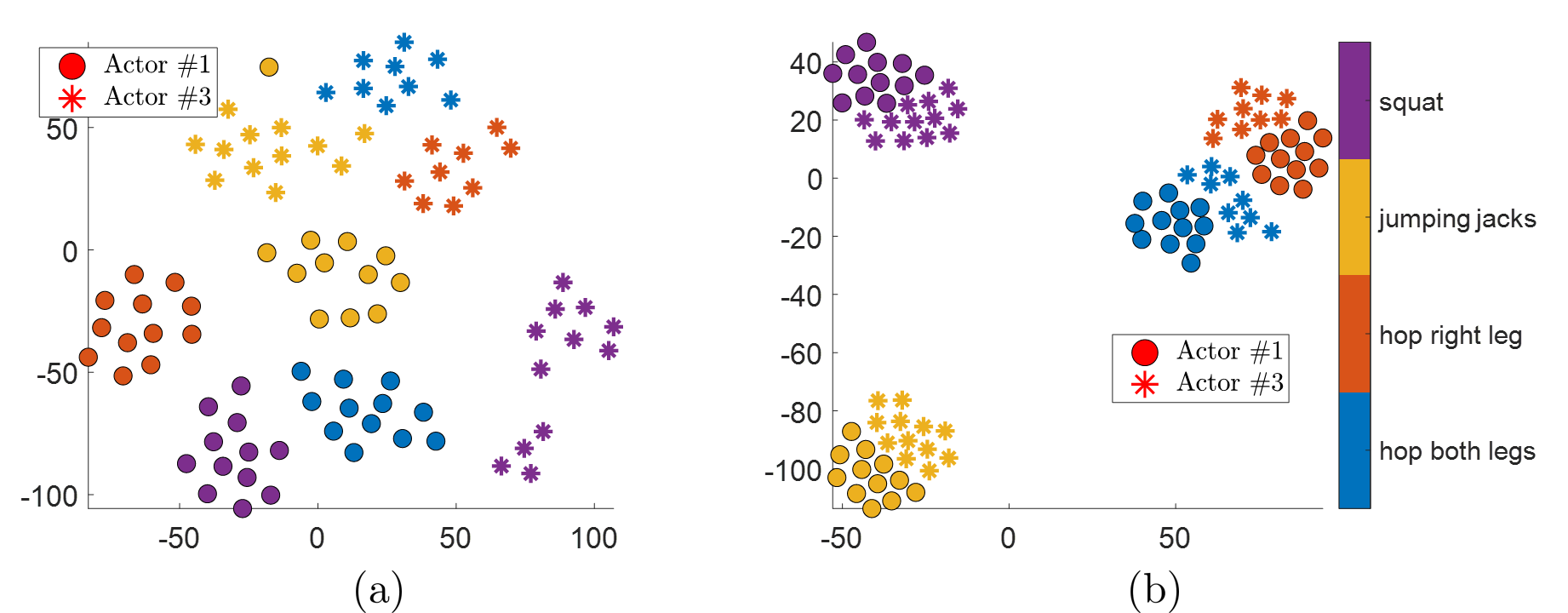}
    \caption{
        (a) 2D tSNE representation of the SPSD matrices $\left\{ \boldsymbol{X}_{i}^{\left(1\right)}\right\} _{i=1}^{48}\cup\left\{ \boldsymbol{X}_{i}^{\left(3\right)}\right\} _{i=1}^{42}$
        (b) 2D tSNE representation of the SPSD matrices $\left\{ \widetilde{\boldsymbol{X}}_{i}^{\left(1\right)}\right\} _{i=1}^{48}\cup\left\{ {\boldsymbol{X}}_{i}^{\left(3\right)}\right\} _{i=1}^{42}$, where $\widetilde{\boldsymbol{X}}_{i}^{\left(1\right)}$ are the SPSD matrics obtained by \cref{alg:SPSD_DA}
     }
    \label{fig:Sub1and3}
\end{figure}

\begin{figure}
    \centering
    \includegraphics[width=0.8\columnwidth]{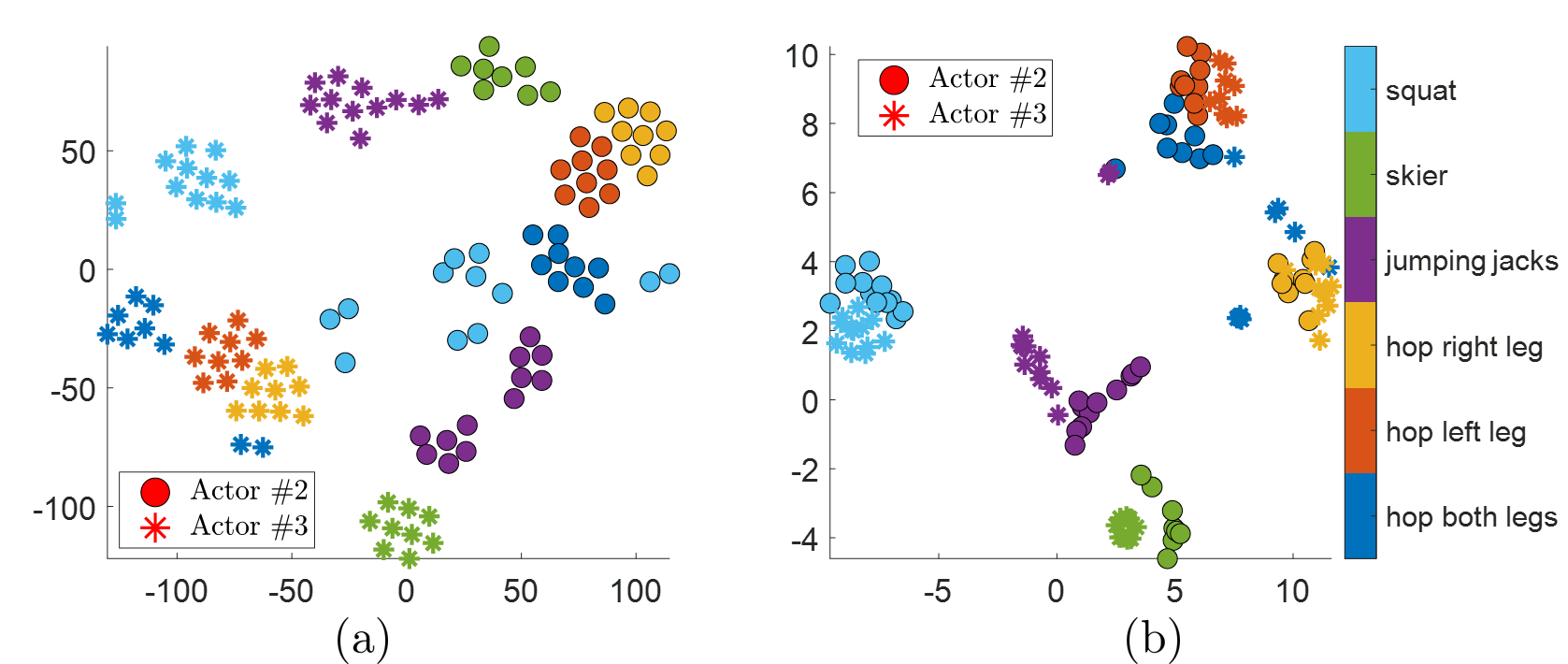}
    \caption{
        (a) 2D tSNE representation of the SPSD matrices $\left\{ \boldsymbol{X}_{i}^{\left(2\right)}\right\} _{i=1}^{59}\cup\left\{ \boldsymbol{X}_{i}^{\left(3\right)}\right\} _{i=1}^{60}$
        (b) 2D tSNE representation of the SPSD matrices $\left\{ \widetilde{\boldsymbol{X}}_{i}^{\left(2\right)}\right\} _{i=1}^{59}\cup\left\{ {\boldsymbol{X}}_{i}^{\left(3\right)}\right\} _{i=1}^{60}$, where $\widetilde{\boldsymbol{X}}_{i}^{\left(2\right)}$ are the SPSD matrices obtained by \cref{alg:SPSD_DA}.
     }
    \label{fig:Sub2and3}
\end{figure}

\begin{table}
\caption{
    Motion recognition classification accuracy.
    (a) Classification without DA. Overall accuracy is 75.91\%
    (b) Classification using \cref{alg:SPSD_DA}. Overall accuracy is 96.73\%.
}
    \label{tab:HDM05accuracy}
\parbox{.495\linewidth}{
\centering
\begin{tabular}{|c|cccc|}
\hline 
{\footnotesize{}\textbf{\backslashbox[1mm]{Train}{Test}}} & \textbf{\footnotesize{}\#1} & \textbf{\footnotesize{}\#2} & \textbf{\footnotesize{}\#3} & \textbf{\footnotesize{}\#5}\tabularnewline
\hline 
\textbf{\footnotesize{}\#1} &  & {\footnotesize{}92.36} & {\footnotesize{}45.83} & {\footnotesize{}96.00}\tabularnewline
\hline 
\textbf{\footnotesize{}\#2} & {\footnotesize{}97.62} &  & {\footnotesize{}95.83} & {\footnotesize{}95.83}\tabularnewline
\hline 
\textbf{\footnotesize{}\#3} & {\footnotesize{}29.17} & {\footnotesize{}91.67} &  & {\footnotesize{}87.50}\tabularnewline
\hline 
\textbf{\footnotesize{}\#5} & {\footnotesize{}98.75} & {\footnotesize{}95.83} & {\footnotesize{}83.33} & \tabularnewline
\hline 
\end{tabular}
\caption*{(a)}

}
\hfill
\parbox{.495\linewidth}{
\centering
\begin{tabular}{|c|cccc|}
\hline 
{\footnotesize{}\textbf{\backslashbox{Train}{Test}}} & \textbf{\footnotesize{}\#1} & \textbf{\footnotesize{}\#2} & \textbf{\footnotesize{}\#3} & \textbf{\footnotesize{}\#5}\tabularnewline
\hline 
\textbf{\footnotesize{}\#1} &  & {\footnotesize{}97.92} & {\footnotesize{}86.11} & {\footnotesize{}96.00}\tabularnewline
\hline 
\textbf{\footnotesize{}\#2} & {\footnotesize{}100} &  & {\footnotesize{}95.83} & {\footnotesize{}100}\tabularnewline
\hline 
\textbf{\footnotesize{}\#3} & {\footnotesize{}100} & {\footnotesize{}100} &  & {\footnotesize{}100}\tabularnewline
\hline 
\textbf{\footnotesize{}\#5} & {\footnotesize{}85.00} & {\footnotesize{}100} & {\footnotesize{}100} & \tabularnewline
\hline 
\end{tabular}
\caption*{(b)}
}
\end{table}

\section{Conclusions}
\label{sec:conclusions}

Data analysis techniques using Riemannian geometry have proven to be useful in a broad range of fields.
In this work, we extend existing results on the Riemannian geometry of SPSD matrices and establish a convenient framework for developing data analysis methods that rely on SPSD matrices as the data features. 
Notable examples for such features are (low rank) covariance matrices, various kernel matrices, and graph Laplacians.
We demonstrate the usefulness of this framework and propose an algorithm for DA using PT on the manifold of SPSD matrices.
We test the algorithm on two applications, hyper-spectral image fusion and motion recognition, and observe good performance.

While the present work follows common practice and the experimental study focuses on covariance matrices, we intend in future work to examine other SPSD matrices.
Perhaps the most significant future direction is the examination of graph Laplacians, which are inherently fixed-rank SPSD matrices and facilitate the representation of entire graphs as data features.

\section*{Acknowledgments}
The motion capture data used in this project was obtained from HDM05 \cite{cg-2007-2}. We thank Wolfgang Gross for making the HSI data from \cite{gross2019nonlinear} available.

\bibliographystyle{siamplain}
\bibliography{references}

\begin{thebibliography}{10}

\bibitem{absil2004riemannian}
{\sc P.-A. Absil, R.~Mahony, and R.~Sepulchre}, {\em Riemannian geometry of
  grassmann manifolds with a view on algorithmic computation}, Acta Applicandae
  Mathematica, 80 (2004), pp.~199--220.

\bibitem{bacak2016second}
{\sc M.~Bac{\'a}k, R.~Bergmann, G.~Steidl, and A.~Weinmann}, {\em A second
  order nonsmooth variational model for restoring manifold-valued images}, SIAM
  Journal on Scientific Computing, 38 (2016), pp.~A567--A597.

\bibitem{barachant2013classification}
{\sc A.~Barachant, S.~Bonnet, M.~Congedo, and C.~Jutten}, {\em Classification
  of covariance matrices using a riemannian-based kernel for bci applications},
  Neurocomputing, 112 (2013), pp.~172--178.

\bibitem{bergmann2018priors}
{\sc R.~Bergmann, J.~H. Fitschen, J.~Persch, and G.~Steidl}, {\em Priors with
  coupled first and second order differences for manifold-valued image
  processing}, Journal of mathematical imaging and vision, 60 (2018),
  pp.~1459--1481.

\bibitem{bhatia2009positive}
{\sc R.~Bhatia}, {\em Positive definite matrices}, vol.~24, Princeton
  university press, 2009.

\bibitem{bonnabel2013rank}
{\sc S.~Bonnabel, A.~Collard, and R.~Sepulchre}, {\em Rank-preserving geometric
  means of positive semi-definite matrices}, Linear Algebra and its
  Applications, 438 (2013), pp.~3202--3216.

\bibitem{bonnabel2010riemannian}
{\sc S.~Bonnabel and R.~Sepulchre}, {\em Riemannian metric and geometric mean
  for positive semidefinite matrices of fixed rank}, SIAM Journal on Matrix
  Analysis and Applications, 31 (2010), pp.~1055--1070.

\bibitem{cohen1960coefficient}
{\sc J.~Cohen}, {\em A coefficient of agreement for nominal scales},
  Educational and psychological measurement, 20 (1960), pp.~37--46.

\bibitem{courty2014domain}
{\sc N.~Courty, R.~Flamary, and D.~Tuia}, {\em Domain adaptation with
  regularized optimal transport}, in Joint European Conference on Machine
  Learning and Knowledge Discovery in Databases, Springer, 2014, pp.~274--289.

\bibitem{courty2016optimal}
{\sc N.~Courty, R.~Flamary, D.~Tuia, and A.~Rakotomamonjy}, {\em Optimal
  transport for domain adaptation}, IEEE transactions on pattern analysis and
  machine intelligence, 39 (2016), pp.~1853--1865.

\bibitem{deng2018active}
{\sc C.~Deng, X.~Liu, C.~Li, and D.~Tao}, {\em Active multi-kernel domain
  adaptation for hyperspectral image classification}, Pattern Recognition, 77
  (2018), pp.~306--315.

\bibitem{edelman1998geometry}
{\sc A.~Edelman, T.~A. Arias, and S.~T. Smith}, {\em The geometry of algorithms
  with orthogonality constraints}, SIAM journal on Matrix Analysis and
  Applications, 20 (1998), pp.~303--353.

\bibitem{fang2018new}
{\sc L.~Fang, N.~He, S.~Li, A.~J. Plaza, and J.~Plaza}, {\em A new
  spatial--spectral feature extraction method for hyperspectral images using
  local covariance matrix representation}, IEEE Transactions on Geoscience and
  Remote Sensing, 56 (2018), pp.~3534--3546.

\bibitem{gross2019nonlinear}
{\sc W.~Gross, D.~Tuia, U.~Soergel, and W.~Middelmann}, {\em Nonlinear feature
  normalization for hyperspectral domain adaptation and mitigation of nonlinear
  effects}, IEEE Transactions on Geoscience and Remote Sensing, 57 (2019),
  pp.~5975--5990.

\bibitem{HyperspectralLowRank}
{\sc A.~{Halimi}, P.~{Honeine}, M.~{Kharouf}, C.~{Richard}, and
  J.~{Tourneret}}, {\em Estimating the intrinsic dimension of hyperspectral
  images using a noise-whitened eigengap approach}, IEEE Transactions on
  Geoscience and Remote Sensing, 54 (2016), pp.~3811--3821.

\bibitem{huang2015log}
{\sc Z.~Huang, R.~Wang, S.~Shan, X.~Li, and X.~Chen}, {\em Log-euclidean metric
  learning on symmetric positive definite manifold with application to image
  set classification}, in International conference on machine learning, 2015,
  pp.~720--729.

\bibitem{iordache2011sparse}
{\sc M.-D. Iordache, J.~M. Bioucas-Dias, and A.~Plaza}, {\em Sparse unmixing of
  hyperspectral data}, IEEE Transactions on Geoscience and Remote Sensing, 49
  (2011), pp.~2014--2039.

\bibitem{jayasumana2015kernel}
{\sc S.~Jayasumana, R.~Hartley, M.~Salzmann, H.~Li, and M.~Harandi}, {\em
  Kernel methods on riemannian manifolds with gaussian rbf kernels}, IEEE
  transactions on pattern analysis and machine intelligence, 37 (2015),
  pp.~2464--2477.

\bibitem{kang2017hyperspectral}
{\sc X.~Kang, X.~Zhang, S.~Li, K.~Li, J.~Li, and J.~A. Benediktsson}, {\em
  Hyperspectral anomaly detection with attribute and edge-preserving filters},
  IEEE Transactions on Geoscience and Remote Sensing, 55 (2017),
  pp.~5600--5611.

\bibitem{kapur2016gene}
{\sc A.~Kapur, K.~Marwah, and G.~Alterovitz}, {\em Gene expression prediction
  using low-rank matrix completion}, BMC bioinformatics, 17 (2016), p.~243.

\bibitem{maaten2008visualizing}
{\sc L.~v.~d. Maaten and G.~Hinton}, {\em Visualizing data using t-sne},
  Journal of machine learning research, 9 (2008), pp.~2579--2605.

\bibitem{maman2019domain}
{\sc G.~Maman, O.~Yair, D.~Eytan, and R.~Talmon}, {\em Domain adaptation using
  riemannian geometry of spd matrices}, in ICASSP 2019-2019 IEEE International
  Conference on Acoustics, Speech and Signal Processing (ICASSP), IEEE, 2019,
  pp.~4464--4468.

\bibitem{moakher2005differential}
{\sc M.~Moakher}, {\em A differential geometric approach to the geometric mean
  of symmetric positive-definite matrices}, SIAM Journal on Matrix Analysis and
  Applications, 26 (2005), pp.~735--747.

\bibitem{cg-2007-2}
{\sc M.~M\"{u}ller, T.~R\"{o}der, M.~Clausen, B.~Eberhardt, B.~Kr\"{u}ger, and
  A.~Weber}, {\em Documentation mocap database hdm05}, Tech. Report CG-2007-2,
  Universit\"{a}t Bonn, June 2007.

\bibitem{niu2016hyperspectral}
{\sc Y.~Niu and B.~Wang}, {\em Hyperspectral anomaly detection based on
  low-rank representation and learned dictionary}, Remote Sensing, 8 (2016),
  p.~289.

\bibitem{pennec2006riemannian}
{\sc X.~Pennec, P.~Fillard, and N.~Ayache}, {\em A riemannian framework for
  tensor computing}, International Journal of computer vision, 66 (2006),
  pp.~41--66.

\bibitem{rodrigues2017dimensionality}
{\sc P.~Rodrigues, F.~Bouchard, M.~Congedo, and C.~Jutten}, {\em Dimensionality
  reduction for bci classification using riemannian geometry}, 2017.

\bibitem{rodrigues2018riemannian}
{\sc P.~L.~C. Rodrigues, C.~Jutten, and M.~Congedo}, {\em Riemannian procrustes
  analysis: transfer learning for brain--computer interfaces}, IEEE
  Transactions on Biomedical Engineering, 66 (2018), pp.~2390--2401.

\bibitem{shrivastava2014unsupervised}
{\sc A.~Shrivastava, S.~Shekhar, and V.~M. Patel}, {\em Unsupervised domain
  adaptation using parallel transport on grassmann manifold}, in IEEE winter
  conference on applications of computer vision, IEEE, 2014, pp.~277--284.

\bibitem{wang2012covariance}
{\sc R.~Wang, H.~Guo, L.~S. Davis, and Q.~Dai}, {\em Covariance discriminative
  learning: A natural and efficient approach to image set classification}, in
  2012 IEEE Conference on Computer Vision and Pattern Recognition, IEEE, 2012,
  pp.~2496--2503.

\bibitem{wu2017kernel}
{\sc C.~Wu, L.~Zhang, and B.~Du}, {\em Kernel slow feature analysis for scene
  change detection}, IEEE Transactions on Geoscience and Remote Sensing, 55
  (2017), pp.~2367--2384.

\bibitem{yair2019parallel}
{\sc O.~Yair, M.~Ben-Chen, and R.~Talmon}, {\em Parallel transport on the cone
  manifold of spd matrices for domain adaptation}, IEEE Transactions on Signal
  Processing, 67 (2019), pp.~1797--1811.

\bibitem{yair2019optimal}
{\sc O.~Yair, F.~Dietrich, R.~Talmon, and I.~G. Kevrekidis}, {\em Optimal
  transport on the manifold of spd matrices for domain adaptation}, arXiv
  preprint arXiv:1906.00616,  (2019).

\bibitem{zhang2013hyperspectral}
{\sc H.~Zhang, W.~He, L.~Zhang, H.~Shen, and Q.~Yuan}, {\em Hyperspectral image
  restoration using low-rank matrix recovery}, IEEE Transactions on Geoscience
  and Remote Sensing, 52 (2013), pp.~4729--4743.

\end{thebibliography}

\clearpage

\section{Supplementary materials}
\begin{proposition}
Let $\mathcal{X}=\left\{ \boldsymbol{P}_{i}\right\} _{i}$ be a set of points on $\mathcal{P}_d$ with the Riemannian mean $\overline{\boldsymbol{P}}$.
Consider the map $t:\mathcal{P}_d \rightarrow \mathcal{P}_d$ defined by
\[
\boldsymbol{R}_{i}=t(\boldsymbol{P}_i)=\boldsymbol{T}\boldsymbol{P}_{i}\boldsymbol{T}^{T}
\]
where $\boldsymbol{T}\in \mathrm{GL}_{d}$.
Let $\overline{\boldsymbol{R}}$ be the Riemannian mean of the resulting set $\left\{ \boldsymbol{R}_{i}\right\}_i$ .
The following holds
\[
\Gamma_{\overline{\boldsymbol{P}}\rightarrow\overline{\boldsymbol{R}}}^{+}\left(\boldsymbol{P}_{i}\right)=\boldsymbol{R}_{i},\qquad\forall i
\]
if and only if $\boldsymbol{T}$ is of the form $\boldsymbol{T}=\overline{\boldsymbol{P}}^{\frac{1}{2}}\boldsymbol{B}\overline{\boldsymbol{P}}^{-\frac{1}{2}}$
where either $\boldsymbol{B}\succ0$ or $\boldsymbol{B}\prec0$.
\end{proposition}

\begin{proof}
Using the congruence invariance property of the geometric mean (see \cite{bhatia2009positive}), we have:
\[
\overline{\boldsymbol{R}}=\boldsymbol{T}\overline{\boldsymbol{P}}\boldsymbol{T}^{T}
\]
Note that
\begin{align*}
\boldsymbol{E} & =\left(\overline{\boldsymbol{R}}\overline{\boldsymbol{P}}^{-1}\right)^{\frac{1}{2}}\\
 & =\left(\boldsymbol{T}\overline{\boldsymbol{P}}\boldsymbol{T}^{T}\overline{\boldsymbol{P}}^{-1}\right)^{\frac{1}{2}}\\
 & =\left(\overline{\boldsymbol{P}}^{\frac{1}{2}}\overline{\boldsymbol{P}}^{-\frac{1}{2}}\boldsymbol{T}\overline{\boldsymbol{P}}\boldsymbol{T}^{T}\overline{\boldsymbol{P}}^{-\frac{1}{2}}\overline{\boldsymbol{P}}^{-\frac{1}{2}}\right)^{\frac{1}{2}}\\
 & =\overline{\boldsymbol{P}}^{\frac{1}{2}}\left(\overline{\boldsymbol{P}}^{-\frac{1}{2}}\boldsymbol{T}\overline{\boldsymbol{P}}\boldsymbol{T}^{T}\overline{\boldsymbol{P}}^{-\frac{1}{2}}\right)^{\frac{1}{2}}\overline{\boldsymbol{P}}^{-\frac{1}{2}}
\end{align*}
\paragraph{First direction} Assume $\boldsymbol{T}=\overline{\boldsymbol{P}}^{\frac{1}{2}}\boldsymbol{B}\overline{\boldsymbol{P}}^{-\frac{1}{2}}$,
then:
\begin{align*}
\boldsymbol{E} & =\overline{\boldsymbol{P}}^{\frac{1}{2}}\left(\overline{\boldsymbol{P}}^{-\frac{1}{2}}\boldsymbol{T}\overline{\boldsymbol{P}}\boldsymbol{T}^{T}\overline{\boldsymbol{P}}^{-\frac{1}{2}}\right)^{\frac{1}{2}}\overline{\boldsymbol{P}}^{-\frac{1}{2}}\\
 & =\overline{\boldsymbol{P}}^{\frac{1}{2}}\left(\overline{\boldsymbol{P}}^{-\frac{1}{2}}\overline{\boldsymbol{P}}^{\frac{1}{2}}\boldsymbol{B}\overline{\boldsymbol{P}}^{-\frac{1}{2}}\overline{\boldsymbol{P}}\overline{\boldsymbol{P}}^{-\frac{1}{2}}\boldsymbol{B}\overline{\boldsymbol{P}}^{\frac{1}{2}}\overline{\boldsymbol{P}}^{-\frac{1}{2}}\right)^{\frac{1}{2}}\overline{\boldsymbol{P}}^{-\frac{1}{2}}\\
 & =\overline{\boldsymbol{P}}^{\frac{1}{2}}\left(\boldsymbol{B}\boldsymbol{B}\right)^{\frac{1}{2}}\overline{\boldsymbol{P}}^{-\frac{1}{2}}\\
 & =\pm\boldsymbol{T}
\end{align*}
There will be a $+$ sign if $\boldsymbol{B}\succ0$ and a $-$ sign
if $\boldsymbol{B}\prec0$.
Hence:
\[
\Gamma_{\overline{\boldsymbol{P}}\rightarrow\overline{\boldsymbol{R}}}^{+}\left(\boldsymbol{P}_{i}\right)=\boldsymbol{E}\boldsymbol{P}_{i}\boldsymbol{E}^{T}=\left(\pm\boldsymbol{T}\right)\boldsymbol{P}_{i}\left(\pm\boldsymbol{T}\right)^{T}=\boldsymbol{R}_{i}
\]
\paragraph{Second direction} First, recall that
\[
\boldsymbol{E}=\overline{\boldsymbol{P}}^{\frac{1}{2}}\left(\overline{\boldsymbol{P}}^{-\frac{1}{2}}\boldsymbol{T}\overline{\boldsymbol{P}}\boldsymbol{T}^{T}\overline{\boldsymbol{P}}^{-\frac{1}{2}}\right)^{\frac{1}{2}}\overline{\boldsymbol{P}}^{-\frac{1}{2}}
\]
and let $\boldsymbol{B}\coloneqq\left(\overline{\boldsymbol{P}}^{-\frac{1}{2}}\boldsymbol{T}\overline{\boldsymbol{P}}\boldsymbol{T}^{T}\overline{\boldsymbol{P}}^{-\frac{1}{2}}\right)^{\frac{1}{2}}\succ0$
thus:
\[
\boldsymbol{E}=\overline{\boldsymbol{P}}^{\frac{1}{2}}\boldsymbol{B}\overline{\boldsymbol{P}}^{-\frac{1}{2}}.
\]
Hence, it is enough to show that $$\boldsymbol{T}=\pm\boldsymbol{E}=\overline{\boldsymbol{P}}^{\frac{1}{2}}\left(\pm\boldsymbol{B}\right)\overline{\boldsymbol{P}}^{-\frac{1}{2}}.$$
Assume the mapping is exact, that is:
\[
\boldsymbol{E}\boldsymbol{P}_{i}\boldsymbol{E}^{T}=\boldsymbol{T}\boldsymbol{P}_{i}\boldsymbol{T}^{T},\qquad\forall\boldsymbol{P}_{i}\in\mathcal{X}
\]
then, without loss of generality, we consider $\boldsymbol{P}_{i}=\boldsymbol{I}\in\mathcal{X}$
for some $i$, and thus:
\[
\boldsymbol{E}\boldsymbol{E}^{T}=\boldsymbol{T}\boldsymbol{T}^{T}
\]
leading to
\[
\boldsymbol{I}=\boldsymbol{E}^{-1}\boldsymbol{T}\boldsymbol{T}^{T}\boldsymbol{E}^{-T}=\left(\boldsymbol{E}^{-1}\boldsymbol{T}\right)\left(\boldsymbol{E}^{-1}\boldsymbol{T}\right)^{T}
\]
which implies that $\boldsymbol{E}^{-1}\boldsymbol{T=}\boldsymbol{U}$
is unitary. Now, from
\[
\boldsymbol{E}\boldsymbol{P}_{i}\boldsymbol{E}^{T}  =\boldsymbol{T}\boldsymbol{P}_{i}\boldsymbol{T}^{T},\qquad\forall\boldsymbol{P}_{i}\in\mathcal{X}
\]
we have
\[
\boldsymbol{P}_{i}=\left(\boldsymbol{E}^{-1}\boldsymbol{T}\right)\boldsymbol{P}_{i}\left(\boldsymbol{E}^{-1}\boldsymbol{T}\right)^{T}=\boldsymbol{U}\boldsymbol{P}_{i}\boldsymbol{U}^{T},\qquad\forall\boldsymbol{P}_{i}\in\mathcal{X}
\]
Again, without loss of generality, assume there exists $\boldsymbol{P}_{j}\in\mathcal{X}$
with unique eigenvalues, and thus, also with unique eigenvectors (up to
a sign).
Let $\boldsymbol{v}$ and $\lambda$ be an eigenvector and its corresponding
eigenvalue, such that:
\[
\boldsymbol{P}_{j}\boldsymbol{v}=\lambda\boldsymbol{v}
\]
Since $\boldsymbol{P}_{i}=\boldsymbol{U}\boldsymbol{P}_{i}\boldsymbol{U}^{T}$
for all $\boldsymbol{P}_{i}\in\mathcal{X}$ we have
\[
\implies\boldsymbol{U}\boldsymbol{P}_{j}\boldsymbol{U}^{T}\boldsymbol{v}=\lambda\boldsymbol{v}
\]
and thus:
\[
\boldsymbol{P}_{j}\boldsymbol{U}^{T}\boldsymbol{v}=\lambda\boldsymbol{U}^{T}\boldsymbol{v}
\]
Since, the eigenvectors are unique, we have:
\[
\boldsymbol{v}=\pm\boldsymbol{U}^{T}\boldsymbol{v}
\]
Since this is true for the all the eigenvectors of $\boldsymbol{P}_{j}$
we have:
\[
\boldsymbol{U}=\pm\boldsymbol{I}
\]
So that
\[
\boldsymbol{E}^{-1}\boldsymbol{T=}\pm\boldsymbol{I}
\]
which can be recast as
\[
\boldsymbol{T}=\pm\boldsymbol{E}
\]
\end{proof}

\begin{proposition}
    
Let $\overline{\boldsymbol{Q}} \in \left[\overline{\boldsymbol{Q}}\right]$ and $\overline{\boldsymbol{V}}\in\left[\overline{\boldsymbol{V}}\right]$
be two points in $\mathcal{O}_{d}$, such that $\overline{\boldsymbol{V}}=\Pi_{\overline{\boldsymbol{Q}}}\left(\overline{\boldsymbol{V}}\right)$.
Define $\Gamma_{\overline{\boldsymbol{Q}}\to\overline{\boldsymbol{V}}}^{+}:\mathcal{G}_{d,r}\to\mathcal{G}_{d,r}$ by
\begin{equation}
\Gamma_{\overline{\boldsymbol{Q}}\rightarrow\overline{\boldsymbol{V}}}^{+}\left(\boldsymbol{Q}_{i}\right)=\mathrm{Exp}_{\overline{\boldsymbol{V}}}\left(\Gamma_{\overline{\boldsymbol{Q}}\rightarrow\overline{\boldsymbol{V}}}\left(\mathrm{Log}_{\overline{\boldsymbol{Q}}}\left(\boldsymbol{Q}_{i}\right)\right)\right)
\end{equation}
Then
\begin{equation}
\Gamma_{\overline{\boldsymbol{Q}}\rightarrow\overline{\boldsymbol{V}}}^{+}\left(\boldsymbol{Q}_{i}\right)\sim\Gamma_{\overline{\boldsymbol{Q}}\rightarrow\overline{\boldsymbol{V}}}\left(\boldsymbol{Q}_{i}\right)=\overline{\boldsymbol{V}}\overline{\boldsymbol{Q}}^{T}\boldsymbol{Q}_{i}
\end{equation}
where $\sim$ is the equivalent class,
and if $\boldsymbol{Q}_{i}$ is chosen such that $\boldsymbol{Q}_{i}=\Pi_{\overline{\boldsymbol{Q}}}\left(\boldsymbol{Q}_{i}\right)$, then the equivalence become equality:
\begin{equation}
\Gamma_{\overline{\boldsymbol{Q}}\rightarrow\overline{\boldsymbol{V}}}^{+}\left(\boldsymbol{Q}_{i}\right)=\Gamma_{\overline{\boldsymbol{Q}}\rightarrow\overline{\boldsymbol{V}}}\left(\boldsymbol{Q}_{i}\right)=\overline{\boldsymbol{V}}\overline{\boldsymbol{Q}}^{T}\boldsymbol{Q}_{i}
\end{equation}
\end{proposition}

\begin{proof}
Let 
\[
\overline{\boldsymbol{Q}}\boldsymbol{B}^{\text{skew}}=\text{Log}_{\overline{\boldsymbol{Q}}}\left(\boldsymbol{Q}_{i}\right)\in\mathcal{T}_{\overline{\boldsymbol{Q}}}\mathcal{G}_{d,r}
\]
Thus
\begin{align*}
\Gamma_{\overline{\boldsymbol{Q}}\rightarrow\overline{\boldsymbol{V}}}^{+}\left(\boldsymbol{Q}_{i}\right) & =\mathrm{Exp}_{\overline{\boldsymbol{V}}}\left(\Gamma_{\overline{\boldsymbol{Q}}\rightarrow\overline{\boldsymbol{V}}}\left(\mathrm{Log}_{\overline{\boldsymbol{Q}}}\left(\boldsymbol{Q}_{i}\right)\right)\right)\\
 & =\mathrm{Exp}_{\overline{\boldsymbol{V}}}\left(\Gamma_{\overline{\boldsymbol{Q}}\rightarrow\overline{\boldsymbol{V}}}\left(\overline{\boldsymbol{Q}}\boldsymbol{B}^{\text{skew}}\right)\right)\\
 & =\mathrm{Exp}_{\overline{\boldsymbol{V}}}\left(\overline{\boldsymbol{V}}\boldsymbol{B}^{\text{skew}}\right)\\
 & =\overline{\boldsymbol{V}}\text{exp}\left(\boldsymbol{B}^{\text{skew}}\right)\\
 & =\overline{\boldsymbol{V}}\overline{\boldsymbol{Q}}^{T}\underbrace{\overline{\boldsymbol{Q}}\text{exp}\left(\boldsymbol{B}^{\text{skew}}\right)}_{\sim\boldsymbol{Q}_{i}}\\
 & \sim\overline{\boldsymbol{V}}\overline{\boldsymbol{Q}}^{T}\boldsymbol{Q}_{i}
\end{align*}
\end{proof}

\end{document}